\documentclass[11pt]{article}
\pdfoutput=1

\usepackage{ifthen}
\usepackage{bbm}
\usepackage{caption}

\usepackage[
  backend=biber,
  style=alphabetic,
  maxalphanames=12,
  maxcitenames=2,
  maxbibnames=99,
  giveninits=false,
  doi=false,
  url=false,
  natbib=true
]{biblatex}

\DeclareDelimFormat[bib]{nametitledelim}{\addspace}

\usepackage{bm}
\usepackage{fullpage}
\usepackage[english]{babel}
\usepackage[utf8]{inputenc}
\usepackage[T1]{fontenc}
\usepackage{lmodern}
\usepackage{enumitem}
\setlist{nosep}
\usepackage{microtype}
\usepackage{booktabs}
\usepackage{hhline}
\usepackage{makecell}
\usepackage{etoolbox}
\apptocmd{\sloppy}{\hbadness 10000\relax}{}{}

\usepackage{amsmath,amsthm,amssymb,mathtools}
\usepackage{graphicx}
\usepackage[textsize=tiny]{todonotes}
\usepackage[colorlinks=true, allcolors=blue]{hyperref}
\usepackage{algorithm,algorithmicx}
\usepackage[noend]{algpseudocode}
\usepackage{xfrac}
\usepackage[normalem]{ulem}

\hypersetup{colorlinks,
            breaklinks=true,
            linkcolor=blue,
            citecolor=blue,
            urlcolor=magenta,
            linktocpage,
            plainpages=false,
            bookmarks=false}

\usepackage{notations}
\usepackage{comment}
\usepackage{tikz,pgfplots,pgfplotstable}
\usepgfplotslibrary{units}
\usepackage{tabularx} 
\usepackage{subfig}
\usepackage{array}
\usepgfplotslibrary{groupplots}

\title{Efficient Agnostic Learning of Univariate Densities \\ Under Local Differential Privacy}

\numberwithin{equation}{section}

\begin{document}

\author{
    Mohammad Afzali\\[3pt]
    {\small University of British Columbia}\\
    {\small\normalfont moafzali@cs.ubc.ca}
    \and
    Hassan Ashtiani\\[3pt]
    {\small McMaster University} \\  {\small Vector Institute}\\
    {\small\normalfont zokaeiam@mcmaster.ca}
    \and
    Sivan Sabato\\[3pt]
    {\small McMaster University}\\
    {\small Canada AI CIFAR Chair, Vector Institute}\\
    {\small\normalfont sabatos@mcmaster.ca}
}

\maketitle

\begin{abstract}%
We propose a general framework for non-interactive locally
differentially private (LDP) learning of structured univariate
distributions that can be approximated by piece-wise polynomials.
This gives the first sample-efficient LDP learners for several
distribution classes, including mixtures of
log-concave distributions, such as Gaussian mixture models
(GMMs). These learners also provide agnostic guarantees and
are computationally efficient.
We also prove a lower bound for learning GMMs under LDP, improving on the standard pure-DP packing lower bounds of \citet{bun2019private}. 
To obtain our general learning framework, we develop an efficient non-interactive LDP algorithm for synthetic data generation over large domains, which may be of independent interest. We then use it together with a deterministic adaptation of the non-private piecewise polynomial learner of \citet{acharya2017sample}.  
\end{abstract}

\allowdisplaybreaks

\section{Introduction}

Distribution learning is a fundamental problem in statistics and machine learning, where the goal is to estimate an unknown probability distribution $P$ from i.i.d. samples~\citep{devroye2001combinatorial}. In the \emph{realizable setting}, $P$ is assumed to belong to a known class $\cF$, such as Gaussian distributions, Gaussian mixtures, or log-concave distributions. The learner seeks to return a distribution $\widehat P$ that approximates $P$ in total variation distance. 
In practice, the chosen class may only approximate $P$, leading to model misspecification.  
\emph{Agnostic distribution learning} captures this scenario by allowing $P$ to be arbitrary and measuring the learner's error relative to the best approximation to $P$ within $\cF$~\citep{yatracos1985rates,devroye2001combinatorial,lai2016agnostic,diakonikolas2019robust,ashtiani2020near,diakonikolas2023algorithmic}. Our paper is concerned with computationally efficient agnostic distribution learning for structured distributions.

In a typical learning scenario, each data point may be contributed by an individual. Maintaining privacy of those individuals then introduces a further challenge: although distribution learning aims to characterize a population, the samples may reveal sensitive information about the individuals. \emph{Differential privacy} provides a rigorous framework for limiting the effect of any individual's data on the distribution of the learner's output~\citep{dwork2006calibrating}. We seek distribution learners that combine a form of this protection with agnostic accuracy guarantees.

In this work, we consider the \emph{local differential privacy} (LDP) model~\citep{warner1965randomized,evfimievski2003limiting,kasiviswanathan2011what, duchi2013local}. Unlike the central model of privacy, the local model removes the need for individuals/users to entrust their raw observations to a data collector: each user sends a randomized report to the server based on their data point, and the server aggregates these privatized reports. We focus on \emph{non-interactive} protocols, in which the local randomization protocols are fixed in advance (as opposed to interactive protocols where the randomization protocol can be adjusted based on the reports collected so far).
Non-interactive LDP protocols are easier to implement at large scales and have been adopted by major technology companies, including Google~\citep{erlingsson2014rappor} and Apple~\citep{apple2017learning}.

It has been observed that a wide range of structured univariate distributions---such as Gaussian mixture models and more generally mixtures of log-concave distributions---can be well approximated by piecewise polynomials~\cite{chan2014efficient}. Without privacy, piecewise polynomials can be learned with sample-and-time-efficient algorithms~\cite{chan2014efficient,acharya2017sample}. These algorithms are also agnostic, and therefore can be readily used to learn structured distributions. This paradigm has been extended to the central differential privacy setting by \citet{diakonikolas2015differentially}. However, there is no known algorithm that works in the \emph{local} model of privacy. This raises the following questions:

\begin{quote}
    Is there a sample efficient learning algorithm for learning structured distributions under LDP? Can this algorithm be made agnostic, time-efficient, and non-interactive?
\end{quote}

\paragraph{Summary of our contribution}
In this work we answer the above questions positively. We give the first general framework for sample-efficient LDP learning of structured univariate distributions that can be approximated by piece-wise polynomials. Our learning algorithm is non-interactive, provides agnostic guarantees and is computationally efficient. Our framework applies to a broad range of distribution classes that admit suitable piece-wise polynomial approximations. As applications, we obtain the first sample-efficient non-interactive LDP learners for bounded-parameter Gaussian mixture models (GMMs) and mixtures of log-concave distributions. 

\paragraph{Paper organization}

We start by formally defining the problem setup in Section~\ref{sec:setting}. We then provide a summary of our main results and contributions in Section~\ref{sec:main-res} and discuss related work in Section~\ref{sec:related-work}. Next, in Section~\ref{sec:prelims}, we present the preliminaries. Sections~\ref{sec:synthetic-generation} and~\ref{sec:main-agnostic-learner} develop private synthetic data generation and our generic agnostic LDP learner, respectively. Section~\ref{sec:apps-learning-mixtures} applies the framework to Gaussian and log-concave mixtures, and Section~\ref{sec:lower-bound} establishes sample complexity lower bounds. Omitted proofs and further discussion on related work appear in the appendix.

\section{Locally private agnostic learning of distributions}\label{sec:setting}

In this section, we introduce formal notions required to introduce the problem setup. A locally private mechanism/algorithm can access data points only  through local randomizers.

\begin{definition}[Local randomizer] Let $\cX,\cY$ be two domains, and $\eps>0$ be a parameter. A randomized function $\cM:\cX\rightarrow \cY$ is called an $\eps$-local randomizer if for all $x,x'\in \cX$, and any measurable subset $Y\subseteq \cY$, it holds that $\bP[\cM(x)\in Y] \leq e^\eps \bP[\cM(x')\in Y]$,
    where the probabilities are over the randomness of $\cM$. 
\end{definition}

\begin{definition}[Non-interactive local differential privacy]
Let $\cX,\cY$ be two domains, $\eps>0$, and $n\in\bN$.
An algorithm $\cA:\cX^n\to\cY$ is non-interactive
$\eps$-locally differentially private ($\eps$-LDP) if
there exist $\eps$-local randomizers $\cM_i:\cX\to\cZ_i$ for
$i\in[n]$, and a (possibly randomized) aggregation function
$\cT:\prod_{i=1}^n\cZ_i\to\cY$, such that $\cA(x_1,\ldots,x_n)=\cT(\cM_1(x_1),\ldots,\cM_n(x_n))$. Each of the local randomizers and the aggregation function have independent randomness.
\end{definition}

We can now state the learning problem. Let $\cF$ be a known class of probability distributions over $\bR$, and let $P$ be an arbitrary unknown distribution over $\bR$. There are $n$ users, each holding one sample, with the samples drawn independently from $P$. Given a privacy parameter $\eps\in(0,1)$, an accuracy parameter $\alpha\in(0,1)$, and a failure probability $\beta\in(0,1)$, the goal is to design a non-interactive $\eps$-LDP algorithm that outputs a probability distribution $\widehat P$ satisfying
\begin{align}
    \dtv(\widehat P,P)
\leq C\inf_{Q\in\cF}\dtv(Q,P)+\alpha \label{def:agnostic-factor}
\end{align}
with probability at least $1-\beta$ over the samples and the randomness of the algorithm, where $C\geq1$ is a universal constant.
We call such an algorithm a $(C,\alpha,\beta)$-agnostic learner for $\cF$. We allow $\widehat P$ to lie outside $\cF$; that is, the learner may be \emph{non-proper}. In the realizable case, where $P\in\cF$, the guarantee reduces to $\dtv(\widehat P,P)\leq\alpha$. Note that the privacy requirement should hold even if the dataset is not generated i.i.d.~by a distribution $P$. 

In general, \emph{unbounded} distribution classes (e.g. unbounded univariate Gaussians) cannot be learned with a uniform sample complexity upper bound under LDP due to packing lower bounds in pure differential privacy~\citet{bun2019private}. Thus, to achieve any meaningful result, we restrict the learning task to classes admitting some ``boundedness'' assumptions on the underlying class parameters (e.g. bounds on mean and variance). 
To measure the density estimation error introduced by such boundedness assumptions, we introduce the following definition.
\begin{definition}[Discretization error]
\label{def:discretization-error}
Let $R,\delta>0$, and let $\cF$ be a class of probability distributions over $\bR$. The $(R,\delta)$-discretization error of class $\cF$ is defined as $\operatorname{Disc\text{-}Err}_{R,\delta}(\cF)
    :=
    \sup_{F\in\cF}
    \|    \operatorname{Disc}_{R,\delta}(F)-F
    \|_1$,
where $\operatorname{Disc}_{R,\delta}$ converts a density into a histogram using a grid of step size $\delta$ on $[-R,R]$ and two additional bins for the tails (see Definition~\ref{def:quantization-map} below).
\end{definition}

Next, we define the class of piece-wise polynomial distributions, which plays a key role as a proxy for efficiently learning structured distributions.

\begin{definition}[Piecewise polynomials]
For integers $s\geq 1$ and $d \geq 0$, let $\cP_{s,d}$ be the set of
nonnegative integrable functions $p:\bR\to\bR_{\geq 0}$ for which $\bR$ can be partitioned
into at most $s$ intervals $\{I_1,\cdots,I_k\}$ (where $k\leq s$), such that for each interval $I_i$, there exists a real polynomial $q_i:\bR\rightarrow\bR$ of degree at most $d$ satisfying $p(x)=q_i(x)$ for all (interior points) $x\in I_i$.
\end{definition}

The next definition measures how well a class can be approximated by piecewise polynomials.
\begin{definition}[Piece-wise polynomial approximation error]
\label{def:poly-approx}
Let $s\geq1$ and $d\geq0$ be integers, and let $\cF$ be a class of probability distributions over $\bR$. The $(s,d)$-piece-wise polynomial approximation error of class $\cF$ is $\operatorname{Poly\text{-}Approx}_{s,d}(\cF)
    :=
    \sup_{F\in\cF}
    \inf_{P\in\cP_{s,d}}
    \|P-F\|_1$.
\end{definition}

\section{Main results and contributions}\label{sec:main-res}

Our main contribution is the introduction of a generic framework for non-interactive LDP learning of structured univariate distributions that can be approximated by piece-wise polynomials. 

\begin{theorem}[Informal version of~\ref{main:cor}]
    Let $\cF$ be a class of univariate probability distributions. If  $\operatorname{Poly\text{-}Approx}_{s,d}(\cF)+ \operatorname{Disc\text{-}Err}_{R,\delta}(\cF)\leq \Theta(\alpha)$, then, $\cF$ admits a computationally efficient, non-interactive $\eps$-LDP $(10,\Theta(\alpha),0.1)$-agnostic learner using
$n=\widetilde{O}\left(
\frac{s^2(d+1)^2}{\eps^2\alpha^2}
\log^3\left(\frac{R}{\delta}\right)\right)$ samples. 
\end{theorem}

This is the first algorithm for this setting whose sample complexity is polynomial in all relevant parameters. Moreover, our learner provides agnostic guarantees and is computationally efficient.

Our learning framework builds on the efficient non-private piece-wise polynomial learner of \citet{acharya2017sample}. A deterministic adaptation of their algorithm---as also observed in \cite{diakonikolas2015differentially} in the context of central privacy---can work with a synthetic dataset whose empirical distribution is close to the true distribution in Kolmogorov distance. The key challenge is to generate such a synthetic dataset privately, while ensuring both sample and computational efficiency.
We show that this task is possible, by presenting the following efficient, non-interactive LDP algorithm for synthetic data generation, which may be of independent interest.

\begin{theorem}[Informal version of~\ref{thm:ldp-synthetic}]
Let $P$ be a univariate probability distribution that satisfies $\operatorname{Disc\text{-}Err}_{R,\delta}(\{P\})\leq \Theta(\alpha)$. Then there exists a computationally efficient, non-interactive $\eps$-LDP algorithm that, given $n=\widetilde{O}\left(
\frac{1}{\eps^2\alpha^2}
\log^3\left(\frac{R}{\delta}\right)
\right)$ samples from $P$, generates a synthetic dataset of size $O(1/\alpha)$ whose empirical distribution approximates $P$ to Kolmogorov accuracy $\Theta(\alpha)$ with probability at least $0.9$.
\end{theorem}

As discussed earlier, time efficiency is measured relative to the optimal sample complexity, which requires efficiency with respect to the underlying domain size. Although existing LDP mechanisms for answering prefix queries are sample efficient, they require processing the entire domain, resulting in running time polynomial in the domain size~\cite{cormode2019answering,edmonds2020power}. To overcome this challenge, we provide a lazy adaptation of the mechanism proposed by \citet{cormode2019answering}, obtaining an algorithm that answers range queries in time polylogarithmic in the domain size. We then use this mechanism as a subroutine to efficiently generate private synthetic data.

Our generic learning framework applies to a broad range of structured distribution classes. As applications, we provide \emph{the first sample-efficient non-interactive LDP learners}---which also provide agnostic guarantees and are computationally efficient---for the following classes:

\begin{theorem}[Log-concave mixtures, informal version of~\ref{thm:ldp-logconcave-mixture-learning}]\label{inf:thm:ldp-logconcave-mixture-learning}
Let $\eps,\alpha\in(0,1)$, let $m\geq1$ be an integer, and let $R\geq0$ and $0<\sigma_{\min}\leq\sigma_{\max}<\infty$. The class of mixtures of $m$ log-concave distributions, each with mean $|\mu|\leq R$ and standard deviation $\sigma\in[\sigma_{\min},\sigma_{\max}]$, admits a computationally efficient, non-interactive $\eps$-LDP $(10,\alpha,0.1)$-agnostic learner with sample complexity
$n=\widetilde{O}\left(
\frac{m^2\log^3 D}{\eps^2\alpha^3}
\right)$
and total running time
$\widetilde{O}\left(
\frac{m^3\log^4 D}{\eps^2\alpha^{9/2}}
\right)$,
where $D=\frac{2(R+\sigma_{\max})}{\sigma_{\min}}$.
\end{theorem}

\begin{theorem}[Gaussian mixtures, informal version of~\ref{thm:ldp-gmm-learning}]\label{inf:thm:ldp-gmm-learning}
Under the same parameter assumptions as in Theorem~\ref{inf:thm:ldp-logconcave-mixture-learning}, the class of mixtures of $m$ Gaussians admits a computationally efficient, non-interactive $\eps$-LDP $(10,\alpha,0.1)$-agnostic learner with sample complexity
$n=\widetilde{O}\left(
\frac{m^2\log^3 D}{\eps^2\alpha^2}
\right)$ and total running time
$\widetilde{O}\left(
\frac{m^3\log^4 D}{\eps^2\alpha^3}
\right)$.\footnote{Gaussian mixtures are a subset of mixtures of log-concave distributions; our specialized Gaussian learner improves the dependence on $\alpha$ in both sample complexity and running time.}
\end{theorem}

We further prove a lower bound showing that restrictions on class parameters, such as bounded mean and variance, are necessary in general for a uniform finite sample complexity bound under LDP. The lower bound provided in the following theorem for non-interactive LDP learning of GMMs holds even in the realizable setting with non-proper outputs. Since every Gaussian distribution is log-concave, this lower bound also applies to mixtures of log-concave distributions.

\begin{theorem}[Gaussian mixture lower bound, informal version of~\ref{thm:lower-bound}]
Let $m\geq2$, $\eps\in(0,1]$, and $\alpha>0$ be sufficiently small. Under the same bounds on component means and standard deviations as in Theorem~\ref{inf:thm:ldp-logconcave-mixture-learning}, if the parameter range
$\Lambda:=\frac{R}{\sigma_{\min}}
+\log\left(\frac{\sigma_{\max}}{\sigma_{\min}}\right)\geq \Omega(m)$ is sufficiently large, then any non-interactive $\eps$-LDP algorithm that learns mixtures of $m$ Gaussians to total variation accuracy $\alpha$ with probability at least $0.9$ requires
$n=\Omega\left(
\frac{m^2}{\eps^2\alpha^2}
\log\left(e+\frac{\Lambda}{m}\right)
\right)$ samples.
\end{theorem}

This matches the dependence of our GMM sample complexity upper bound on $m$, $\eps$, and $\alpha$, up to logarithmic factors. In particular, this improves over the standard packing lower bounds for pure differential privacy, which apply in both the central and local models~\citep{bun2019private}. Indeed, packing for mixtures of $m$ components usually has a size exponential in $m$, which then translate to lower bounds of $\Omega(m)$. We show a stronger lower bound of $\Omega(m^2)$ by reducing the problem of LDP estimation of sparse discrete distributions studied by \citet{acharya2020estimating} to the problem of LDP learning of Gaussian mixtures.

Our framework also improves prior LDP learning results for specialized distribution classes:

\begin{itemize}
    \item \textbf{Univariate Gaussians.}
    Theorem~\ref{inf:thm:ldp-gmm-learning} gives an agnostic learner for univariate Gaussians using
    $\widetilde{O}\left(\frac{\log^3(D)}{\eps^2\alpha^2}\right)$
    samples. This obtains the same sample complexity rates as for realizable Gaussian mean estimation~\citep{gaboardi2019locally,joseph2019locally} up to logarithmic factors, while providing agnostic guarantees for density estimation in total variation distance.

\item \textbf{(Sparse) discrete distributions.}
For learning discrete distributions over a domain of size $D$, \citet{acharya2019hadamard}
achieves the optimal sample complexity of $\Theta\left(\frac{D^2}{\eps^2\alpha^2}\right)$. When the unknown distribution has support of size at most $s$ (i.e., is $s$-sparse),
\citet{acharya2020estimating} gives sharper sample complexity bound of
$\Theta\left(\frac{s^2\max\{\log(D/s),1\}}{\eps^2\alpha^2}\right)$.
Both classes can be embedded into classes of piecewise constant densities. Our framework provides agnostic guarantees for these classes while losing only logarithmic factors in sample complexity.
\end{itemize}

All our density learners above are non-proper. In Section~\ref{sec:proper-learning}, we show how to convert these into proper learners with the same asymptotic sample complexity, at the cost of a larger agnostic approximation factor and potentially inefficient algorithm. In particular, we show that our result achieves an exponential improvement in sample complexity over generic LDP hypothesis selection methods for proper learning of structured distributions.

\section{Related Work}\label{sec:related-work}

A closely related result in private learning of structured distributions is provided by
\citet{diakonikolas2015differentially}. They study this problem in the central differential privacy model and present a connection between private Kolmogorov approximation and structured distribution learning. They state a formal piece-wise constant (histogram) learner and sketch the reduction, explaining that it can be extended to piece-wise polynomials. They separately develop and evaluate a heuristic for private Kolmogorov approximation. However, privatization techniques for Kolmogorov approximation differ substantially between the central and local models. In particular, the central DP guarantee used in their theoretical reduction relies on the approximately differentially private threshold-release algorithm of \citet{bun2015differentially}, which builds on their private algorithm for the interior point problem. 

A closely related line of work studies \emph{locally private hypothesis selection (LPHS)}. Given a finite class $\cF$ of probability distributions and i.i.d. samples from an arbitrary distribution $P$, the goal is to select $\widehat P\in\cF$ whose total variation distance to $P$ is within a constant factor of $\inf_{Q\in\cF}\dtv(P,Q)$, up to an additive error $\alpha$, while satisfying $\eps$-local differential privacy. For this problem, \citet{gopi2020locally} give a non-interactive algorithm with sample complexity
$\widetilde{O}(|\cF|^2/(\eps^2\alpha^2))$, which was later improved by \citet{kamath2025query} to
$\widetilde{O}(|\cF|^{3/2}/(\eps^2\alpha^2))$. A worst-case lower bound of $\Omega(|\cF|/(\eps^2\alpha^2))$ holds even for interactive protocols~\citep{gopi2020locally,pour2024sample}. Thus, without exploiting additional structure in $\cF$, a generic hypothesis-selection method cannot achieve sublinear dependence on $|\cF|$. The running time of these algorithms also scales polynomially with $|\cF|$, as
they typically need to scan the entire class.

These hypothesis selection methods also yield agnostic learners for some infinite distribution classes by applying hypothesis selection to a finite total-variation cover. However, such covers can be exponentially large in terms of class parameters, making this approach inefficient in both sample complexity and running time. For example, discretizing the parameters of a single Gaussian gives polynomially many choices in domain size $D$; enumerating $m$ components and discretizing their weights yields a GMM cover exponential in $m$. Applying hypothesis selection inherits this dependence. On the other hand, our framework exploits distributional structure to obtain sample complexity and running time polynomial in $m$ and polylogarithmic in $D$. See Section~\ref{sec:proper-learning} for further discussion of LPHS and Section~\ref{sec:more-related-work} for more on related work.

\section{Preliminaries}\label{sec:prelims}
For $N\in\bN$, let $[N]:=\{1,2,\ldots,N\}$, and $[N]_0:=\{0,1,2,\ldots,N\}$. The total variation distance between two probability distributions $P$ and $Q$ on $\cX$ is defined as $\dtv(P,Q):=\sup_A|P(A)-Q(A)|=\frac{1}{2}\|P-Q\|_1$, where $A$ ranges over measurable sets and $\|\cdot\|_1$ denotes the total variation norm. For densities $f$ and $g$, this equals $\frac{1}{2}\int_{\cX}|f(x)-g(x)|\,dx$. Let $\log_+(x):=\max\{1,\log x\}$.
For a probability distribution $P$ over $\bR$, we use the notation $F_P$ to denote its CDF.
For two probability distributions $P,Q$ over $\bR$, their Kolmogorov distance is defined as $\|F_P-F_Q\|_\infty
=
\sup_{x\in\bR}|F_P(x)-F_Q(x)|$.

Throughout the paper, we measure running time in a unit-cost RAM model with exact real arithmetic and integer words large enough to store all indices and counters used by our algorithms. The word size is logarithmic in the relevant input and algorithmic parameters, including their inverses when appropriate. Arithmetic operations, comparisons, memory access, shifts, bitwise operations, word-level parity, evaluation of the elementary functions used below, and sampling from the Bernoulli and uniform distributions used by our algorithms all take constant time.

In the following two subsections, we first introduce LDP
primitives for answering range queries and then present the
non-private piece-wise polynomial learner used in our
downstream learning algorithm.

\subsection{Locally private prefix frequencies}

We present tools for answering range-frequency queries over a finite ordered domain under LDP. These queries ask for the fraction of observations lying in a specified interval. We focus on prefix queries, which provide the cumulative frequency information used to generate our private synthetic dataset described in Section~\ref{sec:synthetic-generation}. General interval frequencies can be obtained by taking the difference of two prefix frequencies.

A key ingredient in answering prefix queries efficiently in both samples and time is the \emph{dyadic decomposition}, which partitions each prefix into logarithmically many disjoint intervals from a fixed binary hierarchy. This hierarchy underlies the binary-tree mechanisms for private counting~\citep{dwork2010continual,chan2011private} and their local analogues discussed by \citet{cormode2019answering,edmonds2020power}. Following the approach to locally private range queries of \citet{cormode2019answering}, we combine this hierarchy with the Hadamard randomized response to estimate interval frequencies. We first define the dyadic hierarchy and then introduce the Hadamard characters used in this construction.

\begin{definition}[Dyadic tree representation]
\label{def:dyadic-tree-cdf}
Let $[N]$ be an ordered domain, where $N=2^L$. Then, for every depth $\ell\in[L]_0$ and every
$v\in[2^\ell-1]_0$, the dyadic interval is defined as 
\[
I_{\ell,v}
:=
\left\{
v2^{L-\ell}+1,\ldots,(v+1)2^{L-\ell}
\right\}.
\]
The interval $I_{0,0}=[N]$ is the root of the tree, and the intervals
$I_{L,v}=\{v+1\}$ are its leaves.
For $x\in[N]$, let $v_\ell(x)
:=
\left\lfloor
\frac{x-1}{2^{L-\ell}}
\right\rfloor$
denote the index of the unique depth-$\ell$ node containing $x$.
\end{definition}

If the domain size is not a power of two, we pad it to the next power of two with unused elements.

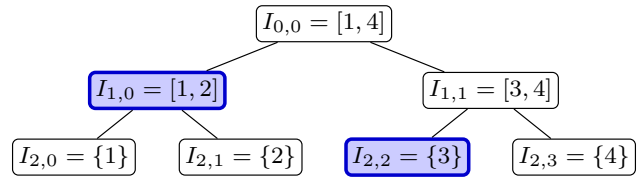
\begin{figure}[htbp]
\centering
\resizebox{0.5\textwidth}{!}{%
\begin{tikzpicture}[
  level distance=8mm,
  level 1/.style={sibling distance=40mm},
  level 2/.style={sibling distance=20mm},
  every node/.style={
    draw,
    rounded corners=2pt,
    inner sep=2pt,
    font=\scriptsize,
    fill=white
  },
  selected/.style={
    fill=blue!20,
    draw=blue!80!black,
    very thick
  },
  edge from parent/.style={draw,-}
]
\node {$I_{0,0}=[1,4]$}
  child {node[selected] {$I_{1,0}=[1,2]$}
    child {node {$I_{2,0}=\{1\}$}}
    child {node {$I_{2,1}=\{2\}$}}
  }
  child {node {$I_{1,1}=[3,4]$}
    child {node[selected] {$I_{2,2}=\{3\}$}}
    child {node {$I_{2,3}=\{4\}$}}
  };
\end{tikzpicture}%
}

\vspace{1mm}
\captionsetup{font=footnotesize}
\caption{For $N=4$, the blue nodes
$\mathcal C(3)=\{I_{1,0},I_{2,2}\}$ partition the prefix $[1,3]$,
using at most one node per depth.}
\label{fig:dyadic-prefix-cover-cdf}
\end{figure}

The following lemma shows that any proper prefix $[j]$ of $[N]$ can be partitioned into at most $\log_2 N$ dyadic intervals, with at most one interval at each depth, and that this partition can be constructed in $O(\log N)$ time. Summing the frequencies of these disjoint intervals therefore gives the prefix frequency. The full prefix $[N]$ has frequency one and requires no estimation. See Figure~\ref{fig:dyadic-prefix-cover-cdf} for an illustration. For completeness, \ref{proofs:sec:synthetic-generation} provides the decomposition algorithm and the proof.

\begin{lemma}[Canonical dyadic decomposition of a prefix]
\label{lem:dyadic-prefix-cover-cdf}
Let $L\in\bN$, $N=2^L$, and let $j\in[N-1]$. Then, Algorithm~\ref{alg:dyadic-prefix-cover-cdf}, upon receiving $j$, returns a collection $\mathcal C(j)$ of dyadic nodes satisfying:
\begin{enumerate}[label=(\arabic*)]
\item the nodes in $\mathcal C(j)$ are pairwise disjoint and $[j]
=
\bigcup_{(\ell,v)\in\mathcal C(j)} I_{\ell,v}$.

\item for every $\ell\in [L]$, the collection
$\mathcal C(j)$ contains at most one node at depth $\ell$; therefore, $|\mathcal C(j)|\leq L$.

\item construction of $\mathcal C(j)$ takes $O(L)$ time.
\end{enumerate}
\end{lemma}

To estimate dyadic interval frequencies, we use the Hadamard
randomized-response construction of \citet{cormode2019answering} (see Algorithm~\ref{alg:hierarchical-ldp-local-randomizer} in \ref{proofs:sec:synthetic-generation}), which uses \emph{Hadamard characters}:

\begin{definition}[Hadamard characters]
\label{def:hadamard-characters-cdf}
For $\ell\in\bN$ and $v\in[2^\ell-1]_0$, the Hadamard
character indexed by $v$ is the function
$\Had_\ell(v,\cdot):[2^\ell-1]_0\to\{-1,+1\}$ defined,
for every $u\in[2^\ell-1]_0$, by
$\Had_\ell(v,u)
    :=
    (-1)^{\langle v,u\rangle_2}$,
where $\langle v,u\rangle_2$ is the binary inner product
of the $\ell$-bit representations of $v$ and $u$.
Each $\Had_\ell(v,u)$ can be computed in $O(1)$ time
using a bitwise AND and a parity operation.
\end{definition}
At a high level, each user helps estimate prefix frequencies by
contributing information at a randomly selected depth $\ell_u$ of the dyadic tree.
To do so, user $u$ randomly chooses a node with index $r_u$
at depth $\ell_u$ and computes $\Had_{\ell_u}(v_u,r_u)$,
where $v_u$ is the index of their true node, namely the node
containing $X_u$ at the same depth. They then apply randomized response to this
sign and report it together with $\ell_u$ and $r_u$.
We show in Section~\ref{subsec:prefix-oracle} how the server
can efficiently reuse each user's report to answer different prefix
queries using the orthogonality of Hadamard characters.

\subsection{A deterministic non-private piecewise polynomial learner}

Another key primitive is the non-private density estimator of \citet{acharya2017sample}. We use a deterministic adaptation of their algorithm. Their original estimator takes i.i.d.~samples from an unknown distribution. As we formalize in Appendix~\ref{appendix:deterministic-adls}, the estimator can instead operate on synthetic data whose empirical distribution satisfies a suitable Kolmogorov-distance bound. We define the class of piecewise polynomial functions and state the deterministic guarantee.

\begin{proposition}[Deterministic post-sampling General--Merging \citep{acharya2017sample}]
\label{prop:deterministic-robust-general-merging}
There exists a universal integer $C>2$ such that the following holds.
Let $I\subset\bR$ be a known finite interval, and let $Z=\{z_1,\ldots,z_M\}\in I^M$ be an arbitrary multiset, where $M\geq1$, with empirical probability measure $H=\frac{1}{M}\sum_{i\in[M]}\delta_{z_i}$.
Let $s\geq1$ and $d\geq0$ be integers, and let $\eta>0$.
There exists a deterministic algorithm which, given $I,Z,s,d,$ and $\eta$, returns a nonnegative function $\widehat P\in\cP_{Cs,d}$ supported on $I$ with the following guarantee:

For every probability distribution $P^\star$ with a density supported on $I$ satisfying
\mbox{$\|F_H-F_{P^\star}\|_\infty\leq\frac{\eta}{2(C+1)(d+1)s}$}, the output of the algorithm satisfies $\|\widehat P-P^\star\|_1
\leq4\bigl(\mathrm{OPT}+\eta\bigr)$,
where $\mathrm{OPT}:=
\inf_{P\in\cP_{s,d}}\|P-P^\star\|_1$.
Moreover, the running time is
$\widetilde{O}\left(
M(d+1)^{\omega+2}
\log_+^3\left(\frac{s}{\eta}\right)
\left(1+\log_+\left(\frac{M}{s}\right)\right)
\right)$,
where $\omega$ is the matrix multiplication exponent, and $\widetilde{O}$ hides only logarithmic factors in $M$ and $d+1$.
\end{proposition}

\section{Efficient locally private synthetic data generation}\label{sec:synthetic-generation}

In this section, we develop a sample- and time-efficient non-interactive LDP algorithm for generating synthetic data with the Kolmogorov approximation guarantee required by the downstream non-private piece-wise polynomial learner. We do this in two steps. First, in Section~\ref{subsec:prefix-oracle}, we develop an efficient LDP oracle for answering prefix queries on a fixed dataset. Then, in Section~\ref{subsec:synthetic-data}, we use this oracle to generate a synthetic dataset whose empirical distribution approximates the discretized input distribution in Kolmogorov distance.

\subsection{Lazy dyadic prefix mechanism}\label{subsec:prefix-oracle}

In this section, we introduce $\mathsf{PrefixQuery}$ (Algorithm~\ref{alg:hierarchical-ldp-prefix-query}), a non-interactive $\eps$-LDP algorithm for estimating prefix frequencies. Given a discrete domain $[D]$, and $n$ users holding values $X_1,\ldots,X_n\in[D]$, the algorithm estimates the fraction of observations in $[j]$ for any queried index $j\in[D]$. Each query takes $O(n+\log D)$ time, and $n=O\left(
\frac{\log^2 D}{\eps^2\alpha^2}
\log\frac{D}{\beta}
\right)$ users suffice to achieve an additive error of at most $\alpha$ simultaneously over all prefixes, with probability at least $1-\beta$ over the randomness of the LDP mechanism.

The algorithm uses the same local randomizer as the binary hierarchical mechanism of \citet{cormode2019answering}, as presented in Algorithm~\ref{alg:hierarchical-ldp-local-randomizer}.
 However, while the server-side reconstruction described by \citet{cormode2019answering} materializes the full hierarchy, requiring $O(n+D\log D)$ time, our server-side implementation evaluates requested prefixes lazily, without constructing the full hierarchy. This gives a running time of $O(q(n+\log D))$ for answering $q$ prefix queries. This is especially useful when the number of queries $q$ is logarithmic in the domain size $D$. As we will see in the following sections, this is indeed the case for the downstream non-private algorithm.

\begin{algorithm}[tbp]
\caption{Hierarchical LDP $\mathsf{PrefixQuery}$}
\label{alg:hierarchical-ldp-prefix-query}
\begin{algorithmic}[1]
\Require Privacy parameter $\eps\in(0,1)$, domain size
$D\in\bN$ with $D\geq2$, a query index $j\in [D]_0$,
and a transcript $\Tr=(T_1,\ldots,T_n)$, where $n\geq1$.
\Ensure An estimate $\widehat Z_j\in[0,1]$ of the prefix frequency
$Z_j(X)=\frac{1}{n}\sum_{u\in[n]}\boldsymbol 1\{X_u\leq j\}$.
\vspace{2mm}

\State If $j=0$ \Return $\widehat Z_j\gets0$.
\State If $j=D$ \Return $\widehat Z_j\gets1$.
\vspace{1mm}

\State Set $L\gets\lceil\log_2D\rceil$, $N\gets2^L$, and $\theta_\eps\gets\frac{e^\eps-1}{e^\eps+1}$.

\State Apply Algorithm~\ref{alg:dyadic-prefix-cover-cdf} (from Lemma~\ref{lem:dyadic-prefix-cover-cdf}) to the prefix index $j$ in the dyadic tree over $[N]$.
Represent the returned nodes as
$\cC(j)=\{(\ell_i,v_i):i\in[k]\}$, where
$[j]=\bigcup_{(\ell_i,v_i)\in\cC(j)}I_{\ell_i,v_i}$,
$k\leq L$, and the depths $\ell_i$ are distinct. Store the nodes of $\cC(j)$ indexed by depth.
\vspace{1mm}

    \State For $u\in[n]$, set
    $
        Z_u\gets
        \frac{L}{\theta_\eps}
        \sum\limits_{\substack{
            (\ell_i,v_i)\in\cC(j)\\
            \ell_i=\ell_u
        }}
        Y_u\Had_{\ell_i}(v_i,r_u)
    $.
\vspace{2mm}

\State \Return
$
    \widehat Z_j\gets
    \min\left\{1,\max\left\{0, \frac{1}{n}\sum\limits_{u\in[n]}Z_u\right\}\right\}
$. \Comment{Average the contributions and clip to $[0,1]$.}
\end{algorithmic}
\end{algorithm}

The following theorem formalizes the utility and privacy guarantees
of $\mathsf{PrefixQuery}$.
\begin{theorem}[Non-interactive LDP prefix oracle]\label{thm:prefix_oracle}
Let $\eps,\alpha,\beta\in (0,1)$, and $D\in \bN$ with $D\geq 2$ be parameters, and let $X=(X_1,\cdots,X_n)\in[D]^n$ be any fixed dataset. Define its prefix frequencies by
\begin{align*}
    Z_j(X)
    :=
    \frac{1}{n}\sum_{u\in[n]}\boldsymbol 1\{X_u\leq j\},
    \qquad j\in[D]_0.
\end{align*}
Let $\cM$ be the local randomizer in Algorithm~\ref{alg:hierarchical-ldp-local-randomizer}, and set the users' transcripts to be $\Tr_{\cM}(X):=(\cM(X_1,\eps,D),\ldots,\cM(X_n,\eps,D))$. For an input query $j\in [D]_0$, let $\widehat Z_j = \mathsf{PrefixQuery}(\eps,D,j,\Tr_{\cM}(X))$ be the output of Algorithm~\ref{alg:hierarchical-ldp-prefix-query}. Then,
\begin{enumerate}[label=(\arabic*)]
    \item The mechanism $\mathsf{PrefixQuery}(\eps,D,j,\Tr_{\cM}(.))$ is non-interactive $\eps$-LDP for every $j\in [D]_0$.
    \item If $n\geq \frac{288(\log D)^2}{\eps^2\alpha^2}\log\left(\frac{2D}{\beta}\right)$, then with probability at least $1-\beta$ over the randomness of $\cM$, it holds that $\left|\widehat Z_j-Z_j(X)\right|\leq \alpha$ for every $j\in[D]_0$ simultaneously.
    \item Generating $\Tr_{\cM}(X)$ takes $O(n)$ total run-time. Each call to $\mathsf{PrefixQuery}$ takes $O(n+\log D)$ server side run-time.
\end{enumerate}
\end{theorem}

At a high level, $\mathsf{PrefixQuery}$ decomposes the queried
prefix $[j]$ into the disjoint dyadic nodes in $\cC(j)$ and
estimates their total frequency. To estimate the frequency of
a node $(\ell_i,v_i)\in\cC(j)$, the server takes user reports with
$\ell_u=\ell_i$ and multiplies the reported sign $Y_u$ by
$\Had_{\ell_i}(v_i,r_u)$. This works because Hadamard
orthogonality (Claim~\ref{claim:hadamard-orthogonality-cdf}) gives
\[
    \bE_{r\sim\operatorname{Unif}([2^{\ell_i}-1]_0)}
    \left[
        \Had_{\ell_i}(v_{\ell_i}(X_u),r)
        \Had_{\ell_i}(v_i,r)
    \right]
    =
    \boldsymbol 1\{X_u\in I_{\ell_i,v_i}\}.
\]
Here, $v_{\ell_i}(X_u)$ is the index of the node containing
$X_u$ at depth $\ell_i$. After rescaling to account for randomized response, the contribution $Z_u$ in
Algorithm~\ref{alg:hierarchical-ldp-prefix-query} satisfies
$\bE[Z_u\mid X_u]=\boldsymbol 1\{X_u\leq j\}$.
Averaging these contributions therefore gives an unbiased
estimate of $Z_j(X)$.
Since orthogonality holds for every node and the reports are
generated without reference to $j$, the server can reuse the
same reports to estimate different prefixes.

Since $\cC(j)$ contains at most one node per depth, each report
takes $O(1)$ time to process, giving $O(n+\log D)$ time per query. This also bounds each user's rescaled contribution $Z_u$, allowing concentration bounds to establish simultaneous accuracy over all prefixes. See~\ref{proofs:sec:synthetic-generation} in the appendix for the formal proof.

\subsection{Efficient LDP synthetic dataset generation}\label{subsec:synthetic-data}

In this section, we develop an efficient non-interactive LDP algorithm for generating a synthetic dataset whose empirical distribution approximates the discretized input distribution in Kolmogorov distance. We first introduce a discretization that maps real-valued samples to a finite ordered set of bin indices, allowing us to apply the prefix oracle from Section~\ref{subsec:prefix-oracle}. We then use this oracle to locate approximate quantiles and place synthetic points at these locations. Although the distributions considered in our applications may have unbounded support, suitable parameter bounds allow us to control both the tail mass and the approximation error introduced by binning.

\begin{definition}[Uniform discretization]\label{def:quantization-map}
Let $R,\delta>0$ be parameters,  and $D=\lceil2R/\delta\rceil+2$. Let $c_0= -R-\delta$, $c_D=R+\delta$, and set $c_j= -R+\min\{(j-1)\delta,2R\}$ for $j\in [D-1]$. For a probability distribution $P$ over $\bR$, define the discretized distribution \mbox{$\operatorname{Disc}_{R,\delta}(P)
    :=
    \sum_{j=1}^D P(B_j)\operatorname{Unif}\left([c_{j-1},c_j)\right)$},
where $B_1=(-\infty,-R), B_D=[R,+\infty)$, and $B_j=[c_{j-1},c_j)$ for $j\in\{2,\cdots,D-1\}$. By a slight abuse of notation, for a point $x\in\bR$, we write
    $\operatorname{Disc}_{R,\delta}(x)=j$
    iff $x\in B_j$, which denotes the index of the interval containing $x$. Each endpoint $c_j$ and the index of the interval containing a given $x$ can be computed in constant time.
\end{definition}

\begin{algorithm}[tbp]
\caption{Non-interactive LDP synthetic data generation}
\label{alg:synthetic-dataset}
\begin{algorithmic}[1]
\Require Parameters $\eps,\alpha \in (0,1)$, $R,\delta>0$, and $n$ users holding values
$X_1,\ldots,X_n\in \bR$.

\State Set $D\gets\lceil2R/\delta\rceil+2$ and $M\gets\left\lceil 1/\alpha\right\rceil$, and let the endpoints $c_j$ be as in Definition~\ref{def:quantization-map}.

\State Each user computes $\widetilde X_u\gets\operatorname{Disc}_{R,\delta}(X_u)\in[D]$ locally. Let $\cM$ be the local randomizer in Algorithm~\ref{alg:hierarchical-ldp-local-randomizer}, and set the users' transcripts to be $\Tr=\left(\cM\left(\widetilde X_1,\eps,D\right),\ldots,\cM\left(\widetilde X_n,\eps,D\right)\right)$.

\For{$\ell=1,\ldots,M$}
    \State Set $u_\ell\gets\frac{\ell-1/2}{M}$. \Comment{Target quantile level for the $\ell$-th point.}
    \State Initialize $(a,b)\gets(0,D)$.
    \Comment{Indices of the current search endpoints.}
    \State Initialize $(A,B)\gets(0,1)$.
    \Comment{CDF estimates at $c_a,c_b$; maintain $A<u_\ell\leq B$.}
    \While{$b>a+1$}
        \State Let $j\gets\lfloor(a+b)/2\rfloor$.
        \State Compute $\widehat Z_j\gets\Call{PrefixQuery}{j,\Tr}$. \Comment{Estimated CDF at $c_j$ using Algorithm~\ref{alg:hierarchical-ldp-prefix-query}.}
        \If{$\widehat Z_j<u_\ell$}
            \State $(a,A)\gets(j,\widehat Z_j)$.
        \Else
            \State $(b,B)\gets(j,\widehat Z_j)$.
        \EndIf
    \EndWhile
    \State Set $\lambda_\ell\gets\frac{u_\ell-A}{B-A}$. \Comment{Relative position of $u_\ell$ between $A$ and $B$.}
    \State Set $z_\ell\gets c_{b-1}+\lambda_\ell(c_b-c_{b-1})$. \Comment{Place the synthetic point at this position in bin $b$.}
\EndFor

\State \Return the synthetic dataset $\cZ\gets\{z_1,\ldots,z_M\}$, passed as a multiset.
\end{algorithmic}
\end{algorithm}

The following theorem presents the privacy and utility guarantees of our synthetic data generation algorithm (Algorithm~\ref{alg:synthetic-dataset}). Utility is measured in Kolmogorov distance with respect to the discretized input distribution $\operatorname{Disc}_{R,\delta}(P)$. We later show how choosing $R$ and $\delta$ appropriately for the distribution class of interest allows us to relate this guarantee to the original distribution $P$, using suitable tail bounds and binning arguments.

\begin{theorem}[LDP synthetic dataset]\label{thm:ldp-synthetic}
Let $\eps,\alpha,\beta\in(0,1)$, $R>0$, $\delta\in (0, R/2]$, and $n\in \bN$ be parameters, and let $P$ be an unknown distribution over $\bR$.
Then, the following holds:
\begin{enumerate}[label=(\arabic*)]
    \item For any dataset $X\in\bR^n$, Algorithm~\ref{alg:synthetic-dataset} is non-interactive $\eps$-LDP.
    \item If the dataset $X$ consists of $n$ i.i.d.~samples from $P$, where
$n=\Omega\left(
    \frac{1}{\eps^2\alpha^2}
    \log^2\left(\frac{R}{\delta}\right)
    \log\left(\frac{R}{\beta\delta}\right)
\right)$,
    then with probability at least $1-\beta$ over the sampling and local randomizers, the empirical probability measure $H:=\frac{1}{M}\sum_{i\in[M]}\delta_{z_i}$ of the output, satisfies
    $\left\|F_H-F_{\operatorname{Disc}_{R,\delta}(P)}\right\|_\infty
    \leq\alpha$.
    \item The total run-time of Algorithm~\ref{alg:synthetic-dataset}, including user preprocessing and server side computation, is
    $O\left(
            \frac{1}{\alpha}
            \log\left(\frac{R}{\delta}\right)
            \left(n+\log\left(\frac{R}{\delta}\right)\right)
        \right)$.
\end{enumerate}
\end{theorem}

Let $\widetilde P:=\operatorname{Disc}_{R,\delta}(P)$. At a high level, Algorithm~\ref{alg:synthetic-dataset} represents $\widetilde P$ by $M$ equally weighted points, one near each quantile level $u_\ell$. Using the estimates $\widehat Z_j$ returned by $\mathsf{PrefixQuery}$
in Algorithm~\ref{alg:hierarchical-ldp-prefix-query}, we perform
binary search to find adjacent bin endpoints whose estimated
CDF values $A,B$ satisfy $A<u_\ell\leq B$. The interpolation weight $\lambda_\ell$ then places $z_\ell$ within this bin so that its estimated cumulative probability matches $u_\ell$.  

Since $F_{\widetilde P}$ is linear within each bin, the errors in the endpoint estimates translate directly into errors between $F_{\widetilde P}(z_\ell)$ and $u_\ell$. Since the quantile levels $u_\ell$ are evenly spaced and each point $z_\ell$ has mass $1/M$ in $H$, these errors control how closely $F_H$ follows $F_{\widetilde P}$, with an additional error of order $1/M$ from using finitely many points. All searches reuse the same transcript and require only $O(\log D)$ prefix queries per point. Appendix~\ref{proofs:sec:synthetic-generation} provides the formal proof.

\section{Efficient LDP agnostic piece-wise polynomial learner}\label{sec:main-agnostic-learner}

In this section, we combine the ingredients provided above to obtain an efficient LDP agnostic piece-wise polynomial learner. We feed the private synthetic dataset (from section~\ref{subsec:synthetic-data}) into the deterministic adaptation of the non-private piece-wise polynomial learner of \citet{acharya2017sample} (See Proposition~\ref{prop:deterministic-robust-general-merging}). To relate the resulting guarantee to the unknown distribution, we account for the error introduced by discretization (Definition~\ref{def:discretization-error}), as well as the error from approximating the target class by piece-wise polynomials (Definition~\ref{def:poly-approx}).

Putting these ingredients together gives our general efficient LDP agnostic learning framework theorem. See~\ref{proofs:sec:main-agnostic-learner} in the appendix for the proof.
\begin{theorem}\label{main:cor}
Let $\eps,\alpha,\beta\in (0,1)$, $R>0$, $0<\delta\leq R/2$ be parameters, let $s\geq 1,d\geq 0$ be integers, let $P$ be an unknown distribution over $\bR$, and let $\cF$ be a known class of probability distributions over $\bR$.
Then, there exist a universal integer $C\geq1$ and a
non-interactive $\eps$-LDP algorithm that, using
$n=\Theta\left(
    \frac{s^2(d+1)^2}{\eps^2\alpha^2}
    \log^2\left(\frac{R}{\delta}\right)
    \log\left(\frac{R}{\beta\delta}\right)
\right)$
i.i.d.~samples from $P$, outputs a distribution $\widehat P$
with a nonnegative piece-wise polynomial density in $\cP_{Cs,d}$
satisfying:
\begin{align*}
    \|\widehat P-P\|_1
    \leq
    10\mathrm{OPT}
    +8\alpha
    +9\operatorname{Disc\text{-}Err}_{R,\delta}(\cF)
    +8\operatorname{Poly\text{-}Approx}_{s,d}(\cF)
\end{align*}
with probability at least $1-\beta$ over sampling and local randomizers, where $\mathrm{OPT}:=\inf_{F\in\cF}\|F-P\|_1$. Moreover, the total running time, including user preprocessing and server side computation, is $\widetilde{O}\left(
\frac{s(d+1)}{\alpha}
\left(
n\log\left(\frac{R}{\delta}\right)
+(d+1)^{\omega+2}
\right)
\right)$,
where $\omega$ denotes the matrix multiplication exponent, and $\widetilde{O}$ hides only logarithmic factors in $s$, $d+1$, and $1/\alpha$.
\end{theorem}

\section{Efficient LDP Agnostic Learning of Mixtures of Log-concave distributions}\label{sec:apps-learning-mixtures}

In this section, as an application of our framework, we provide
an efficient non-interactive LDP agnostic learner for mixtures
of log-concave distributions with bounded component parameters. The class of mixtures of log-concave distributions is quite
expressive and includes GMMs and mixtures of uniform, exponential,
Laplace, and logistic distributions. It also allows mixtures whose
components belong to different distribution families.
Restrictions on the component parameter ranges are necessary
for a uniform finite bound on sample complexity in the
non-interactive LDP setting, as we demonstrate below by the packing
lower bound for GMMs in Theorem~\ref{thm:lower-bound}. All proofs omitted from this section are given in Section~\ref{proofs:sec:apps-learning-mixtures}.

\subsection{Mixture distributions}\label{subsection:mixtures}

Here, we define mixture distributions and show that bounding their polynomial approximation and discretization errors reduces to bounding the corresponding errors for the component class. 

\begin{definition}[Mixture class] \label{def:k-mixtures}
    Let $\cQ$ be an arbitrary class of distributions. We denote the class of $m$-mixtures of $\cQ$ by $\mmix(\cQ)= \Delta_m \times \cQ^m$ where $\Delta_m = \left\{ w \in \bR^m \,:\, w_i \geq 0, \sum_{i\in[m]} w_i = 1 \right\}$ is the $(m-1)$-dimensional probability simplex.
\end{definition}

\begin{lemma}[Discretization and polynomial approximation of mixtures]
\label{lem:mixture-errors}
Let $\cF$ be a nonempty class of probability distributions over $\bR$ with densities. Let $m,s\geq1$ and $d\geq0$ be integers, and let $R,\delta>0$. Then,
\begin{enumerate}[label=(\arabic*)]
    \item $\operatorname{Disc\text{-}Err}_{R,\delta}
    \left(\mmix(\cF)\right)=
    \operatorname{Disc\text{-}Err}_{R,\delta}(\cF)$,

    \item $\operatorname{Poly\text{-}Approx}_{m(s-1)+1,d}
    \left(\mmix(\cF)\right) \leq
    \operatorname{Poly\text{-}Approx}_{s,d}(\cF)$.
\end{enumerate}
\end{lemma}

\subsection{An upper bound for mixtures of log-concave distributions}
\label{sec:efficient-nonproper-logconcave}

The formal definition of bounded-parameter log-concave distributions is given as follows.
\begin{definition} [Log-concave functions]
\label{def:logconcave-function}
    A function $f:\bR\to\bR_{\geq0}$ is log-concave if, for every $x,y\in\bR$ and $\lambda\in(0,1)$, it holds that
    $f(\lambda x+(1-\lambda)y)
        \geq f(x)^\lambda f(y)^{1-\lambda}$.
\end{definition}

\begin{definition} [Bounded-Parameter Log-concave Distributions]
\label{def:bounded-logconcave}
    Let $R\geq0$, and $0<\sigma_{\min}\leq\sigma_{\max}<\infty$ be parameters. Then the class of bounded-parameter log-concave distributions over $\bR$ is defined by $\cL(R,\sigma_{\min},\sigma_{\max})
        =\left\{f: f\text{ is a log-concave probability density},
        |\mu_f|\leq R,
        \sigma_f\in[\sigma_{\min},\sigma_{\max}]\right\}$,
    where $\mu_f=\bE_{X\sim f}[X]$, and $\sigma_f^2=\bE_{X\sim f}[(X-\mu_f)^2]$.
\end{definition}

The following theorem gives the resulting guarantee for efficient non-interactive LDP agnostic learning of log-concave mixtures. 

\begin{theorem}[Efficient LDP Agnostic Learning of Log-concave Mixtures]
\label{thm:ldp-logconcave-mixture-learning}
Let $\eps,\alpha,\beta\in(0,1)$, let $m\geq1$ be an integer, and let $R\geq0$ and
$0<\sigma_{\min}\leq\sigma_{\max}<\infty$ be parameters.
Then, there exists a non-interactive $\eps$-LDP $(10,\alpha,\beta)$-agnostic learner for the class $\mmix(\cL(R,\sigma_{\min},\sigma_{\max}))$, that uses $
    n
    =\widetilde{O}\left(
        \frac{m^2}{\eps^2\alpha^3}
        \log^2\left(
            \frac{2(R+\sigma_{\max})}{\alpha\sigma_{\min}}
        \right)
        \log\left(
            \frac{2(R+\sigma_{\max})}{\alpha\beta\sigma_{\min}}
        \right)
    \right)$,
samples and its output 
has a nonnegative piece-wise linear density with $\widetilde{O}(m/\sqrt{\alpha})$ pieces.
Moreover, the total running time is
$\widetilde{O}\left(
        \frac{nm}{\alpha^{3/2}}
        \log\left(
            \frac{2(R+\sigma_{\max})}{\alpha\sigma_{\min}}
        \right)
    \right)$.
Here, $\widetilde{O}$ hides only logarithmic factors in $1/\alpha$ in the sample and piece-count bounds, and in $m$ and $1/\alpha$ in the running-time bound.
\end{theorem}

To obtain this result, it suffices to bound the polynomial
approximation and discretization errors of the single-component
class $\cL(R,\sigma_{\min},\sigma_{\max})$. The following two
lemmas provide these bounds, which extend to the mixture class
by Lemma~\ref{lem:mixture-errors}, allowing us to apply the
LDP piece-wise polynomial learner from Theorem~\ref{main:cor}.

For log-concave distributions, the piece-wise linear approximation result of \citet{chan2014efficient} gives the required polynomial approximation bound. We state it below in terms of our class approximation error. 
\begin{lemma}[Piece-wise polynomial approximation of log-concave distributions]
\label{lem:logconcave-polynomial-approximation}
For every $\alpha\in(0,1)$, there exists an integer $s_\alpha=\widetilde{O}(1/\sqrt{\alpha})$ such that, for every $R\geq0$ and $0<\sigma_{\min}\leq\sigma_{\max}<\infty$, it holds that 
    $\operatorname{Poly\text{-}Approx}_{s_\alpha,1}
    \left(\cL(R,\sigma_{\min},\sigma_{\max})\right)
    \leq\alpha$.
Here, $\widetilde{O}$ hides only logarithmic factors in $1/\alpha$.
\end{lemma}

Standard tail bounds for log-concave distributions control the
error from truncating the domain. Within the retained intervals, we use unimodality and a bound on the density to control the error introduced by discretization. Combining these arguments gives the following lemma.
\begin{lemma}[Discretization of Bounded-Parameter Log-concave Distributions]
\label{lem:logconcave-mixture-partition-reconstruction}
Let $\eta\in(0,1)$, $R\geq0$, and $0<\sigma_{\min}\leq\sigma_{\max}<\infty$ be parameters. Let $T=R+\sigma_{\max}\left(1+\log\left(\frac{4}{\eta}\right)\right)$, $J=\left\lceil\frac{8T}{\eta\sigma_{\min}}\right\rceil$, and $\delta=\frac{2T}{J}$.
Then, it holds that
$\operatorname{Disc\text{-}Err}_{T,\delta}
    \left(
        \cL(R,\sigma_{\min},\sigma_{\max})
    \right)
    \leq\eta$.
\end{lemma}

Together with the piece-wise linear approximation bound above, this allows us to apply Theorem~\ref{main:cor} and obtain the efficient non-interactive $\eps$-LDP agnostic learner stated in Theorem~\ref{thm:ldp-logconcave-mixture-learning}.

\section*{AI Disclosure} 
The authors initially worked with the factorization mechanism of \citet{edmonds2020power} for answering prefix frequency queries. While the authors were looking for ways to make their algorithm time-efficient, AI (GPT-5.6 Sol) highlighted the connection to the local randomizer of \citet{cormode2019answering}. The authors also used AI to check proofs and identify edge cases in algorithms and theorem statements, which they subsequently addressed. All text, theorems, and proofs were written by the authors. AI assistance with writing was limited to language and grammar refinement.

\section*{Acknowledgments}

Mohammad Afzali was supported in part by a Four Year Doctoral
Fellowship (4YF) from the University of British Columbia.
Sivan Sabato holds a Canada AI CIFAR Chair at the Vector Institute. Hassan Ashtiani is also a faculty member at the Vector Institute. We acknowledge the support of the Natural Sciences and Engineering Research Council of Canada (NSERC), funding reference numbers RGPIN-2024-05907 and RGPIN-2019-06018.
Resources used in preparing this research were provided,
in part, by the Province of Ontario, the Government of Canada through CIFAR, and companies sponsoring the Vector Institute; see \url{https://vectorinstitute.ai/partnerships/current-partners/}.

\printbibliography

\appendix

\section{More on the related work}
\label{sec:more-related-work}
In this section, we discuss further related work and its connections to our framework.

In the realizable setting, prior work has studied nonparametric
density estimation under LDP under smoothness assumptions. \citet{duchi2018minimax} study Sobolev density estimation under squared $L_2$ loss, while \citet{butucea2020local} study Besov classes under integrated $L_r$ loss, including adaptation to unknown smoothness. Both works assume that the sampling density belongs to a
specified smoothness class on a fixed bounded interval, and give non-interactive estimators with polynomial running time for fixed smoothness and support parameters, assuming efficient basis evaluation. Although their losses imply TV bounds on bounded supports, these guarantees still require the sampling density to belong to the specified smoothness class. Moreover, their rates depend on smoothness norms and support parameters and do not directly recover our dependence on the number of mixture components or our polylogarithmic dependence on the component parameter range. Our framework instead gives an agnostic TV guarantee based on piece-wise polynomial approximation, covering bounded-parameter mixtures with unbounded support or discontinuities. 

\citet{li2023robustness} study Huber contamination, where samples come from $P=(1-\eta)F+\eta G$ and the goal is to recover the clean distribution $F$. Their density-estimation results assume that $F$ belongs to a Sobolev class and that both $F$ and $G$ are supported on $[0,1]$. They give non-interactive estimators with polynomial running time for fixed smoothness parameters, assuming efficient basis evaluation, and minimax guarantees under squared $L_2$ and $L_\infty$ losses. These imply TV bounds after conversion to a density, but do not establish an agnostic guarantee for arbitrary sampling distributions or our dependence on mixture complexity and component parameter ranges. Our guarantee allows arbitrary $P$ and, under the contamination model above with $F\in\cF$, implies $\dtv(\widehat P,F)\leq(C+1)\eta+\alpha$. However, it does not provide their stronger $L_2$ or $L_\infty$ guarantees.

In the central DP model, learning a single Gaussian has been studied extensively, from univariate confidence intervals~\citep{karwa2018finite} to high-dimensional learning and agnostic estimation of unrestricted Gaussians~\citep{kamath2019privately,adenali2021sample}. Further work studies covariance-aware estimation and robustness~\citep{brown2021covariance,liu2022differential,kothari2022private}. Several works provide computationally efficient algorithms for Gaussian mean, covariance, or density estimation~\citep{kamath2022private,ashtiani2022private,alabi2023privately,hopkins2023robustness,brown2023fast,kuditipudi2023pretty}. More recently, \citet{zampetakis2025private} develop a truncation-based framework for private statistical estimation, including efficient Gaussian mean and covariance estimation.

For mixtures, \citet{adenali2021privately} give central DP sample complexity bounds for mixtures of axis-aligned Gaussians, and \citet{afzali2024mixtures} establish polynomial sample complexity for mixtures of arbitrary multivariate Gaussians. \citet{ashtiani2024sample} improve these bounds and, in particular, obtain a realizable univariate learner using $\widetilde{O}(m/\alpha^2+m\log(1/\delta)/(\eps\alpha))$ samples under central $(\eps,\delta)$-DP, where $\delta$ is the privacy failure parameter. In the agnostic setting, \citet{afzali2025agnostic} give a central DP learner for unrestricted multivariate GMMs; their univariate specialization uses $\widetilde{O}((m^2+m\log(1/\delta))/(\eps\alpha^2))$ samples. These unrestricted density-learning algorithms are computationally inefficient. For mixture parameter estimation, \citet{kamath2019differentially} give efficient central DP algorithms under strong separation assumptions. \citet{arbas2023polynomial} allow weaker separation and unbounded parameters, but their resulting algorithm has sample complexity and running time exponential in the number of components $m$, with a polynomial-time guarantee for fixed $m$. The privatization techniques in the central and local models differ substantially: a central algorithm can access the dataset jointly, whereas each user in our non-interactive local algorithms privatizes their own sample independently. Our results provide efficient agnostic learners for bounded-parameter univariate mixtures in this local setting, with only logarithmic factors in the component parameter ranges.

\section{Non-private piece-wise polynomial learner with synthetic data}\label{appendix:deterministic-adls}

Here, we provide a proof sketch for the deterministic adaptation of the non-private piece-wise polynomial learner of \citet{acharya2017sample} that we use in our framework. Their original estimator takes i.i.d.~samples from an unknown distribution, but it can also operate on synthetic data whose empirical distribution satisfies a suitable Kolmogorov-distance bound. We first define the $\cA_k$ norm used in their analysis.

\begin{definition}[$\cA_k$-norm \citep{acharya2017sample}]
\label{def:ak-norm}
Let $I\subseteq\bR$ be an interval, and $k\geq 1$ be an integer. For a finite signed measure $\mu$ on $\bR$, the $\cA_k$ norm restricted to $I$ is defined as:
\[
    \|\mu\|_{\cA_k,I}
    :=
    \sup_{\substack{0\leq r \leq k\\
                     I_1,\ldots,I_r\subseteq I\text{ disjoint intervals}}}
       \sum_{i\in [r]}|\mu(I_i)|.
\]
The intervals may have arbitrary real endpoints. We allow open, closed, half-open, and singleton intervals. When $I=\bR$, we write $\|\mu\|_{\cA_k}$ instead of $\|\mu\|_{\cA_k,\bR}$.
\end{definition}

\begin{lemma}[$\cA_k$-distance from CDF error]
\label{lem:ak-cdf-bound}
Let $P,Q$ be two probability distributions over $\bR$, with CDFs $F_P,F_Q$, and let $k\geq1$ be an integer. Then,
\begin{align*}
    \|P-Q\|_{\cA_k}
    \leq
    k\|P-Q\|_{\cA_1}
    \leq
    2k\|F_P-F_Q\|_\infty.
\end{align*}
\end{lemma}
\begin{proof}
For any interval $I=(a,b]$, we have
\begin{align*}
    |(P-Q)(I)|
    &=
    |(F_P-F_Q)(b)-(F_P-F_Q)(a)|\\
    &\leq
    |(F_P-F_Q)(b)|+|(F_P-F_Q)(a)|
    \leq2\|F_P-F_Q\|_\infty.
\end{align*}
The same bound holds for other endpoint conventions and unbounded intervals by taking the corresponding one-sided limits of the CDFs. Taking the supremum over intervals gives
\begin{align*}
    \|P-Q\|_{\cA_1}
    \leq2\|F_P-F_Q\|_\infty.
\end{align*}
Moreover, for any collection of $r\leq k$ disjoint intervals $I_1,\cdots,I_r$,
\begin{align*}
    \sum_{i\in[r]}|(P-Q)(I_i)|
    \leq
    r\|P-Q\|_{\cA_1}
    \leq
    k\|P-Q\|_{\cA_1}.
\end{align*}
Taking the supremum over all such collections proves the result.
\end{proof}

\begin{proof}[Proof of Proposition~\ref{prop:deterministic-robust-general-merging}]
Fix an arbitrary probability distribution $P^\star$ satisfying the assumptions of the Proposition~\ref{prop:deterministic-robust-general-merging}, and set $k=(C+1)(d+1)s$.
By Lemma~\ref{lem:ak-cdf-bound},
\begin{align*}
    \|H-P^\star\|_{\cA_k}
    &\leq
    2k\|F_H-F_{P^\star}\|_\infty
    \leq
    2(C+1)(d+1)s\cdot\frac{\eta}{2(C+1)(d+1)s}
    =\eta.
\end{align*}
We use the deterministic post-sampling adaptations of Theorem 17 of \citet{acharya2017sample} for $d>0$ and Theorem 10 for $d=0$, with the supplied multiset $Z$ replacing the sampling step.
Their error analyses apply once the required $\cA_k$-distance bound holds, which we have established above. With suitable universal constants and approximation tolerances, they therefore give
\begin{align*}
    \|\widehat P-P^\star\|_1
    \leq4\bigl(\mathrm{OPT}+\eta\bigr).
\end{align*}
The output has at most $Cs$ pieces of degree at most $d$, and the running time follows from the same deterministic algorithms on $M$ input points. The conversion from Kolmogorov distance is used only in the analysis and requires no additional computation.
\end{proof}

\section{Omitted proofs from sections~\ref{sec:prelims} and~\ref{sec:synthetic-generation}}\label{proofs:sec:synthetic-generation}
\begin{algorithm}[htbp]
\caption{Canonical dyadic decomposition of a prefix}
\label{alg:dyadic-prefix-cover-cdf}
\begin{algorithmic}[1]
\Require Domain size $N=2^L$ and a query index $j\in[N-1]$.
\Ensure A collection $\mathcal C(j)$ of pairwise disjoint dyadic nodes whose union is $[j]$.
\State Let $\mathcal V$ be the $2N-1$ nodes of the complete dyadic tree over $[N]$ (See Definition~\ref{def:dyadic-tree-cdf}). Index them as
\[
I_{\ell,v}
=
\{v2^{L-\ell}+1,\ldots,(v+1)2^{L-\ell}\}
\]
for depth $\ell\in\{0,\ldots,L\}$ and $v\in\{0,\ldots,2^\ell-1\}$ .
\State Initialize $\mathcal C(j)\gets\emptyset$ and $s\gets 0$.
\For{each depth $\ell\in\{1,\ldots,L\}$}
    \State $q\gets 2^{L-\ell}$ \Comment{Length of a node at depth $\ell$}
    \If{$\left\lfloor j/q\right\rfloor\bmod 2=1$}
        \State $v\gets s/q$
        \State $\mathcal C(j)\gets\mathcal C(j)\cup\{I_{\ell,v}\}$
        \State $s\gets s+q$
    \EndIf
\EndFor
\State \Return $\mathcal C(j)$
\end{algorithmic}
\end{algorithm}

\begin{proof}[Proof of Lemma~\ref{lem:dyadic-prefix-cover-cdf}]
We prove that the algorithm~\ref{alg:dyadic-prefix-cover-cdf} satisfies the requirements of the lemma.

Let $j=\sum_{\ell\in [L]} b_\ell 2^{L-\ell}$, be the $L$-bit binary representation of $j$, where
 $b_\ell\in\{0,1\}$.
 At depth $\ell$, Algorithm~\ref{alg:dyadic-prefix-cover-cdf} adds a node to $\mathcal C(j)$ when $b_\ell=1$. Let $s_0=0$, and for each $\ell\in [L]$ define
$s_\ell
    :=
    \sum_{r\in [\ell]} b_r2^{L-r}$.
We show that, after processing depth $\ell$, the nodes selected so far are pairwise disjoint and their union is exactly $[1,s_\ell]$. If $b_\ell=0$, the algorithm does not add a node and $s_\ell=s_{\ell-1}$, so the claim holds.
Otherwise, $b_\ell=1$. Let $q=2^{L-\ell}$. Note that at depth $\ell$, the variable $s$ is initially equal to $s_{\ell-1}$, and every term in $s_{\ell-1}$ is divisible by $q$, therefore
$v=\frac{s_{\ell-1}}{q}$
is an integer. So, the node selected by the algorithm is
$I_{\ell,v}
    =
    \{vq+1,\ldots,(v+1)q\} 
    =
    \{s_{\ell-1}+1,\ldots,s_{\ell-1}+q\}$.
This node is disjoint from all previously selected nodes and is immediately adjacent to their union. Since
$s_\ell=s_{\ell-1}+q$,
the union of the selected nodes after depth $\ell$ is exactly $[1,s_\ell]$. In particular at depth $L$, we have
$s_L
    =
    \sum_{\ell\in [L]}b_\ell2^{L-\ell}
    =
    j$.
Consequently, the nodes in $\mathcal C(j)$ are pairwise disjoint and exactly cover $[1,j]$, proving \textit{(1)}.

The algorithm considers each depth $\ell\in[L]$ exactly once and adds at most one node during that iteration. Thus, $\mathcal C(j)$ contains at most one node at every depth, proving \textit{(2)}. It also follows that
$|\mathcal C(j)|
    =
    \sum_{\ell\in [L]} b_\ell
    \leq L$.

Finally, the algorithm performs $L$ iterations, with a constant amount of operations at every iteration. Hence, it constructs $\mathcal C(j)$ in $O(L)$ time, which completes the proof.
\end{proof}

The following is the standard orthogonality identity for Hadamard characters: distinct characters have zero average product, while each character has average squared value one.
\begin{claim}[Orthogonality of Hadamard characters, Theorem~1.5 of \citet{odonnell2014analysis}]\label{claim:hadamard-orthogonality-cdf}
  Let $\ell\in \bN$, and $u,v\in \{0,\cdots, 2^{\ell}-1\}$, then $\Had_{\ell}$ satisfies the orthogonality identity:
  \begin{align}
  \frac{1}{2^\ell}
    \sum_{r=0}^{2^\ell-1}
       \Had_\ell(v,r)
       \Had_\ell(u,r)
    =
    \boldsymbol 1\{v=u\}.
  \end{align}
\end{claim}

\begin{algorithm}[tbp]
\caption{Hierarchical LDP local randomizer $\mathcal M$}
\label{alg:hierarchical-ldp-local-randomizer}
\begin{algorithmic}[1]
\Require A user's value $X_u\in[D]$, privacy parameter $\eps\in(0,1)$, and domain size
$D\in\bN$ with $D\geq2$.
\Ensure A user transcript $T_u=(\ell_u,r_u,Y_u)$.
\vspace{2mm}

\State Set $L\gets\lceil\log_2D\rceil$.
\vspace{1mm}

\State Sample a depth $\ell_u$ uniformly from $[L]$.
\State Sample $r_u$ uniformly from $\{0,\ldots,2^{\ell_u}-1\}$.
\State Set
$
    v_u
    \gets
    v_{\ell_u}(X_u)
    =
    \left\lfloor
        \frac{X_u-1}{2^{L-\ell_u}}
    \right\rfloor
$.
\vspace{1mm}

\State Set
$
    B_u\gets\Had_{\ell_u}(v_u,r_u)\in\{-1,+1\}
$.
\vspace{1mm}

\State Set $p_\eps\gets\frac{e^\eps}{e^\eps+1}$, and
$
    Y_u
    \gets
    \begin{cases}
        B_u, & \text{with probability }p_\eps,\\
        -B_u, & \text{with probability }1-p_\eps.
    \end{cases}
$
\vspace{2mm}

\State \Return $T_u\gets(\ell_u,r_u,Y_u)$.
\end{algorithmic}
\end{algorithm}

\begin{proof}[Proof of Theorem~\ref{thm:prefix_oracle}] Now, we prove the utility and privacy gaurantee of our lazy PrefixQuery Algorithm~\ref{alg:hierarchical-ldp-prefix-query}.

Set $L:=\lceil\log_2D\rceil$. We first prove \textbf{(1)}.
Fix a user $u\in[n]$, and let $T_u:=\cM(X_u,\eps,D)$.
Note that the depth $\ell_u$ and Hadamard coordinate $r_u$
are sampled independently of the input. Therefore, it is enough
to compare the conditional probabilities of the reported bit $Y_u$.
For any $x,x'\in[D]$ and any possible transcript $(\ell,r,y)$:
\begin{align*}
    \frac{\bP\left[T_u=(\ell,r,y)\mid X_u=x\right]}
         {\bP\left[T_u=(\ell,r,y)\mid X_u=x'\right]}
    &=
    \frac{\bP\left[Y_u=y\mid\ell_u=\ell,r_u=r,X_u=x\right]}
         {\bP\left[Y_u=y\mid\ell_u=\ell,r_u=r,X_u=x'\right]}.
\end{align*}
For any $(\ell,r)$, the unperturbed signs corresponding to $x$
and $x'$ are $b=\Had_\ell(v_\ell(x),r)$ and
$b'=\Had_\ell(v_\ell(x'),r)$, respectively.
If $b=b'$, then the two conditional distributions are identical.
If $b=-b'$, then the largest possible likelihood ratio is
$\frac{p_\eps}{1-p_\eps}=e^\eps$, which proves that $\cM$ is
$\eps$-LDP.

Also, note that each user produces its transcript $T_u$ using
only its own input and independent local randomness.
All calls to $\mathsf{PrefixQuery}$ are performed by the server
after the transcript $\Tr_{\cM}(X)=(T_1,\ldots,T_n)$ has been
collected. These calls only post-process $\Tr_{\cM}(X)$,
so the protocol is non-interactive $\eps$-LDP, regardless
of the number of queries.

Now, we prove \textbf{(2)}. The prefixes $j=0$ and $j=D$ are exact by definition. Let $j\in\{1,\cdots,D-1\}$ be any remaining query index, and $x\in[D]$ be any domain point. Also, let $\cC(j)=\{(\ell_i,v_i): i\in[k]\}$ and
\begin{align*}
    Z_u
    :=
    \frac{L}{\theta_\eps}
    \sum_{(\ell_i,v_i)\in\cC(j)}
    \boldsymbol 1\{\ell_u=\ell_i\}Y_u\Had_{\ell_i}(v_i,r_u)
\end{align*}
be defined as in Algorithm~\ref{alg:hierarchical-ldp-prefix-query}.

We show that:
\begin{quote}
    \textbf{(I)} $\bE[Z_u\mid X_u=x]=\boldsymbol 1\{x\leq j\}$;

    \textbf{(II)} $|Z_u|\leq L/\theta_\eps$ almost surely.
\end{quote}
Here, all expectations are over the depth sampling ($\ell_u$), the Hadamard-coordinate sampling ($r_u$), and randomized response ($Y_u$).

Conditioned on $X_u=x$, $\ell_u=\ell$, and $r_u=r$, randomized response satisfies:
\[
    \bE[Y_u\mid X_u=x,\ell_u=\ell,r_u=r]
    =
    \theta_\eps\Had_\ell(v_\ell(x),r).
\]
Since $\ell_u$ is uniform over $\{1,\ldots,L\}$, it follows from Hadamard orthogonality (Claim~\ref{claim:hadamard-orthogonality-cdf}) that
\[
\begin{aligned}
    \bE[Z_u\mid X_u=x]
    &=
    \frac{L}{\theta_\eps}
    \sum_{(\ell_i,v_i)\in\mathcal C(j)}
    \frac{1}{L}\frac{1}{2^{\ell_i}}
    \sum_{r=0}^{2^{\ell_i}-1}
    \theta_\eps\Had_{\ell_i}(v_{\ell_i}(x),r)\Had_{\ell_i}(v_i,r)\\
    &=
    \sum_{(\ell_i,v_i)\in\mathcal C(j)}
    \boldsymbol 1\{v_{\ell_i}(x)=v_i\}.
\end{aligned}
\]
Recall that $v_{\ell_i}(x)=\left\lfloor\frac{x-1}{2^{L-\ell_i}}\right\rfloor$ denotes the index of the unique depth-$\ell_i$ node containing $x$. Since the intervals in $\mathcal C(j)$ are disjoint and their union is $[1,j]$, the last expression equals $\boldsymbol 1\{x\leq j\}$, proving \textbf{(I)}.

By Lemma~\ref{lem:dyadic-prefix-cover-cdf}, the cover contains at most one node at every depth. Since the user samples only one depth, at most one term in the definition of $Z_u$ can be nonzero. Furthermore, $Y_u$ and every Hadamard entry belong to $\{-1,+1\}$. Hence $|Z_u|\leq L/\theta_\eps$, proving \textbf{(II)}.

For the fixed dataset $X$, the variables $Z_1,\cdots,Z_n$ are independent, although they need not be identically distributed. Let $\widetilde Z_j = \dfrac{1}{n} \sum\limits_{u\in [n]}Z_u$. By \textbf{(I)},
\begin{align*}
    \bE\left[\widetilde Z_j\mid X\right]
    &=
    \frac{1}{n}\sum_{u\in[n]}\bE[Z_u\mid X_u]\\
    &=
    \frac{1}{n}\sum_{u\in[n]}\boldsymbol 1\{X_u\leq j\}
    =Z_j(X).
\end{align*}
Using \textbf{(II)} and Hoeffding's inequality for independent bounded random variables implies:
\begin{align*}
    \bP\left[
        |\widetilde Z_j-Z_j(X)|>\alpha\mid X
    \right]
    &\leq
    2\exp\left(
        -\frac{2n^2\alpha^2}{\sum_{u\in[n]}(2L/\theta_\eps)^2}
    \right)\\
    &=
    2\exp\left(
        -\frac{n\theta_\eps^2\alpha^2}{2L^2}
    \right).
\end{align*}
Taking a union bound over the at most $D$ prefix queries results in:
\begin{align*}
    \bP\left[
        \max_{j\in\{1,\cdots,D-1\}}
        |\widetilde Z_j-Z_j(X)|>\alpha\mid X
    \right]
    \leq
    2D\exp\left(
        -\frac{n\theta_\eps^2\alpha^2}{2L^2}
    \right).
\end{align*}
Since $\theta_\eps\geq\eps/4$ for $0<\eps\leq1$ and $L\leq3\log D$ for $D\geq2$, the right-hand side is at most $\beta$ under the stated sample-size condition. Finally, clipping $\widetilde Z_j$ to $[0,1]$ cannot increase its distance from $Z_j(X)\in[0,1]$, which concludes the proof for \textbf{(2)}.

Finally, we prove \textbf{(3)}. Each user constructs and stores its report in $O(1)$ time and $O(1)$ words, so user preprocessing takes $O(n)$ total time. For each prefix query, the server constructs the dyadic cover in $O(L)$ time and scans the $n$ reports in $O(n)$ time, using a cover representation indexed by depth. Thus, one query takes $O(n+\log D)$ time, and $q\geq1$ queries take $O(q(n+\log D))$ total time including user preprocessing, without materializing the full tree or a length-$D$ vector.
\end{proof}

\begin{lemma}[$L_1$ Stability of Discretization]
\label{lem:forward-bound}
Let $R,\delta>0$ be parameters, and let $P$ and $P'$ be two probability distributions over $\bR$. Then,
\begin{align*}
    \left\|
        \operatorname{Disc}_{R,\delta}(P)
        -\operatorname{Disc}_{R,\delta}(P')
    \right\|_1
    \leq
    \|P-P'\|_1.
\end{align*}
\end{lemma}

\begin{proof}
Let $B_j$'s be the intervals from Definition~\ref{def:quantization-map}. Since the intervals $B_j$'s are pairwise disjoint, we have:
\begin{align*}
    \left\|
        \operatorname{Disc}_{R,\delta}(P)
        -\operatorname{Disc}_{R,\delta}(P')
    \right\|_1
    &=
    \sum_{j=1}^D |P(B_j)-P'(B_j)|\\
    &\leq
    \|P-P'\|_1,
\end{align*}
where the inequality follows because $(B_j)_{j=1}^D$ is a measurable partition of $\bR$.
\end{proof}

\begin{proof}[Proof of Theorem~\ref{thm:ldp-synthetic}] Putting together, we are ready to prove the utility and privacy guarantee of our LDP synthetic data generator in Algorithm~\ref{alg:synthetic-dataset}.

Set $D:=\lceil2R/\delta\rceil+2$ and $L:=\lceil\log_2D\rceil$. Since $\delta\leq R/2$, we have $R/\delta\geq2$, and $\frac{2R}{\delta}
    \leq D
    \leq\frac{4R}{\delta}$.
In particular, $\log D=O(\log(R/\delta))$ and $L=O(\log(R/\delta))$. 

Let the real line partitioning $B_1=(-\infty,-R), B_j=[c_{j-1},c_j), B_D=[R,+\infty)$, be as in Definition~\ref{def:quantization-map}, where $c_0= -R-\delta$, $c_j= -R+\min\{(j-1)\delta,2R\}$, $c_D=R+\delta$.

We first prove \textbf{(1)}.
Each user computes the interval index $\widetilde X_u=\operatorname{Disc}_{R,\delta}(X_u)$ locally and applies the randomizer in Algorithm~\ref{alg:hierarchical-ldp-local-randomizer}. For any two real inputs, their indices belong to $[D]$, so Theorem~\ref{thm:prefix_oracle} proves $\eps$-local differential privacy. The protocol is non-interactive, and all remaining steps are post-processing of the transcripts $\Tr$.

Now, we prove \textbf{(2)}.
Let $\widetilde P=\operatorname{Disc}_{R,\delta}(P)$, and write $F_{\widetilde P}(x)= \bP_{X\sim \widetilde P}(X\leq x)$ for its CDF. By Definition~\ref{def:quantization-map}, $F_{\widetilde P}$ is continuous and linear on each interval $[c_{j-1},c_j]$, and
\begin{align*}
    F_{\widetilde P}(c_j)
    =
    \sum_{i=1}^j P(B_i)
    =
    \bP_{X\sim P}\left[\operatorname{Disc}_{R,\delta}(X)\leq j\right],
    \qquad j\in\{0,\cdots,D\}.
\end{align*}
Let $\widetilde X=\left(\widetilde X_1,\cdots,\widetilde X_n\right)$ be the discretized dataset, and define its prefix frequencies by
\begin{align*}
    Z_j\left(\widetilde X\right)
    :=
    \frac{1}{n}\sum_{u\in[n]}\boldsymbol 1\left\{\widetilde X_u\leq j\right\},
    \qquad j\in\{0,\cdots,D\}.
\end{align*}
Let $\alpha':=\alpha/2$. Conditioned on any realization of $\widetilde X$, Theorem~\ref{thm:prefix_oracle}, with accuracy parameter $\alpha'/2$ and failure probability $\beta/2$, gives
\begin{align*}
    \max_{j\in\{0,\cdots,D\}}
    \left|\widehat Z_j-Z_j\left(\widetilde X\right)\right|
    \leq\frac{\alpha'}{2}
\end{align*}
with probability at least $1-\beta/2$ over the local randomizers, provided that
$n
    \geq
    \frac{288(\log D)^2}{\eps^2(\alpha'/2)^2}
    \log\left(\frac{4D}{\beta}\right)$.

Since the prefix-oracle guarantee holds for every fixed dataset, the same probability bound holds over the joint sampling and local randomness.

Furthermore, since $X_1,\cdots,X_n$ are drawn i.i.d. from $P$, we have $\bE_{X\sim P}\left[Z_j\left(\operatorname{Disc}_{R,\delta}(X)\right)\right]=\bE_{X\sim P}\left[Z_j\left(\widetilde X\right)\right]=F_{\widetilde P}(c_j)$. Using Hoeffding's inequality, followed by a union bound over the prefixes, gives
\begin{align*}
    \bP\left[
        \max_{j\in\{0,\cdots,D\}}
        \left|Z_j\left(\widetilde X\right)-F_{\widetilde P}(c_j)\right|>\frac{\alpha'}{2}
    \right]
    &\leq
    2D\exp\left(-\frac{n(\alpha')^2}{2}\right)\leq\frac{\beta}{2}
\end{align*}
 Combining the two bounds, with probability at least $1-\beta$ over the sampling and local randomizers, the triangle inequality gives
\begin{align*}
    \left|\widehat Z_j-F_{\widetilde P}(c_j)\right|
    &\leq
    \left|\widehat Z_j-Z_j\left(\widetilde X\right)\right|
    +\left|Z_j\left(\widetilde X\right)-F_{\widetilde P}(c_j)\right|\leq\frac{\alpha'}{2}+\frac{\alpha'}{2}
    =\alpha'
\end{align*}
for every $j\in\{0,\cdots,D\}$ simultaneously. We condition on this event throughout the rest of the proof. Furthermore, $\widehat Z_0=0$ and $\widehat Z_D=1$, but the $\widehat Z_j$'s are not necessarily monotone.

For each $\ell\in[M]$, the binary search maintains the condition $A=\widehat Z_a<u_\ell\leq\widehat Z_b=B$. At termination, $a=b-1$, so that $B-A>0$ and $\lambda_\ell=(u_\ell-A)/(B-A)\in(0,1]$. Thus, the interpolation step is well-defined, and it holds that
    $u_\ell
    =(1-\lambda_\ell)\widehat Z_{b-1}+\lambda_\ell\widehat Z_b$, and
    $F_{\widetilde P}(z_\ell)
    =(1-\lambda_\ell)F_{\widetilde P}(c_{b-1})+\lambda_\ell F_{\widetilde P}(c_b)$.
Therefore, the simultaneous accuracy bound implies:
\begin{align*}
    |F_{\widetilde P}(z_\ell)-u_\ell|
    &\leq
    (1-\lambda_\ell)|F_{\widetilde P}(c_{b-1})-\widehat Z_{b-1}|
    +\lambda_\ell|F_{\widetilde P}(c_b)-\widehat Z_b|\leq\alpha'.
\end{align*}

Write $F_H$ for the CDF of $H$. Fix any $x\in\bR$. By the preceding accuracy bound and monotonicity of $F_{\widetilde P}$, if $z_\ell\leq x$, then $u_\ell
    \leq F_{\widetilde P}(z_\ell)+\alpha'
    \leq F_{\widetilde P}(x)+\alpha'$.
Similarly, if $z_\ell>x$, then $u_\ell
    \geq F_{\widetilde P}(z_\ell)-\alpha'
    \geq F_{\widetilde P}(x)-\alpha'$.
    
Since $u_\ell=(\ell-1/2)/M$, these inequalities give
\begin{align*}
    F_H(x)
    &\leq
    \frac{1}{M}\sum_{\ell\in[M]}
    \boldsymbol 1\{u_\ell\leq F_{\widetilde P}(x)+\alpha'\}
    \leq F_{\widetilde P}(x)+\alpha'+\frac{1}{2M},\\
    1-F_H(x)
    &\leq
    \frac{1}{M}\sum_{\ell\in[M]}
    \boldsymbol 1\{u_\ell\geq F_{\widetilde P}(x)-\alpha'\}
    \leq 1-F_{\widetilde P}(x)+\alpha'+\frac{1}{2M}.
\end{align*}
Putting together, gives
\begin{align*}
    \|F_H-F_{\widetilde P}\|_\infty
    =\sup_{x\in\bR}|F_H(x)-F_{\widetilde P}(x)|
    \leq \alpha'+\frac{1}{2M}
    \leq\frac{\alpha}{2}+\frac{\alpha}{2}
    =\alpha.
\end{align*}
This argument does not require the estimated prefix values to be monotone.

Finally, to see \textbf{(3)}, computing the interval index and generating the local randomized response each take $O(1)$ time per user. Thus, user preprocessing takes $O(n)$ total time. Generating the synthetic dataset requires $M$ binary searches, each with at most $L$ prefix queries and constant additional work per step. By Theorem~\ref{thm:prefix_oracle}, each query takes $O(n+L)$ time. Thus, the total run-time of the algorithm is
\begin{align*}
    O\left(n+ML(n+L)\right)
    =O\left(
        \frac{1}{\alpha}
        \log\left(\frac{R}{\delta}\right)
        \left(n+\log\left(\frac{R}{\delta}\right)\right)
    \right),
\end{align*}
where we used $M=\lceil1/\alpha\rceil=O(1/\alpha)$ and $L=O(\log(R/\delta))$.
The assumption $\delta\leq R/2$ is used only to simplify the bounds.
\end{proof}

\section{Proof of the main theorem from section~\ref{sec:main-agnostic-learner}}\label{proofs:sec:main-agnostic-learner}

\begin{proof}[Proof of Theorem~\ref{main:cor}] We now prove the guarantees of our LDP agnostic learning algorithm.

Let $C\geq3$ be a universal constant for which Proposition~\ref{prop:deterministic-robust-general-merging} holds, and set
$t=(C+1)(d+1)s$,
    and $M=\left\lceil\frac{2t}{\alpha}\right\rceil$.

Let $\widetilde P:=\operatorname{Disc}_{R,\delta}(P)$ with support $I_{R,\delta}:=[-R-\delta,R+\delta]$.
Using Theorem~\ref{thm:ldp-synthetic}, with accuracy parameter $\alpha/(2t)$, the stated sample size suffices to obtain a synthetic multiset $\cZ\subseteq I_{R,\delta}$ of size $M$, with empirical probability measure $H$, such that
\begin{align}
    \left\|F_H-F_{\widetilde P}\right\|_\infty
    \leq\frac{\alpha}{2t}
    =\frac{\alpha}{2(C+1)(d+1)s} \label{eq:main:1}
\end{align}
with probability at least $1-\beta$. We condition on this event for the rest of the proof.

Run the deterministic algorithm from Proposition~\ref{prop:deterministic-robust-general-merging} on $I_{R,\delta},\cZ,s,d$, with parameter $\eta=\alpha$.
Since $\widetilde P$ has a density supported on $I_{R,\delta}$, the Kolmogorov bound~\eqref{eq:main:1} satisfies the requirement of the proposition for $P^\star=\widetilde P$.
Thus, the output of the deterministic piece-wise polynomial learner $Q\in\cP_{Cs,d}$ is nonnegative and satisfies:
\begin{align}
    \|Q-\widetilde P\|_1
    \leq
    4\inf_{P'\in\cP_{s,d}}\|P'-\widetilde P\|_1+4\alpha. \label{eq:main:2}
\end{align}
Fix any $F\in\cF$. We can write:
\begin{align}
    \|F-\widetilde P\|_1
    &\leq
    \left\|F-\operatorname{Disc}_{R,\delta}(F)\right\|_1
    +\left\|
        \operatorname{Disc}_{R,\delta}(F)
        -\operatorname{Disc}_{R,\delta}(P)
    \right\|_1\\
    &\leq
    \operatorname{Disc\text{-}Err}_{R,\delta}(\cF)
    +\|F-P\|_1,
\end{align}
where the first inequality follows from the triangle inequality,
and the second follows from Definition~\ref{def:discretization-error}
and the fact that $\operatorname{Disc}_{R,\delta}$ is a Markov
operator and does not increase $L_1$ distance
(see Lemma~\ref{lem:forward-bound}).

Consequently, we can write:
\begin{align*}
    \inf_{P'\in\cP_{s,d}}\|P'-\widetilde P\|_1
    &\leq
    \inf_{P'\in\cP_{s,d}}\|P'-F\|_1+\|F-\widetilde P\|_1\\
    &\leq
    \operatorname{Poly\text{-}Approx}_{s,d}(\cF)
    +\operatorname{Disc\text{-}Err}_{R,\delta}(\cF)+\|F-P\|_1.
\end{align*}
Putting together with \eqref{eq:main:2}, we get:
\begin{align}
    \|Q-\widetilde P\|_1
    \leq
    4\alpha+4\operatorname{Poly\text{-}Approx}_{s,d}(\cF)
    +4\operatorname{Disc\text{-}Err}_{R,\delta}(\cF)+4\|F-P\|_1. \label{eq:main:3}
\end{align}

Since $Q$ is a piece-wise polynomial supported on the bounded interval $I_{R,\delta}$, it is integrable. Let $m_0:=\int_{\bR}Q(x)\,dx < \infty$, and define the normalized output distribution by:
\begin{align*}
    \widehat P:=
    \begin{cases}
        Q/m_0,&m_0>0,\\
        \operatorname{Unif}(I_{R,\delta}),&m_0=0.
    \end{cases}
\end{align*}
If $m_0>0$, then
\begin{align}
    \|\widehat P-\widetilde P\|_1
    \leq
    \|\widehat P-Q\|_1+\|Q-\widetilde P\|_1
    =|1-m_0|+\|Q-\widetilde P\|_1\leq2\|Q-\widetilde P\|_1. \label{eq:main:4}
\end{align}
If $m_0=0$, then $\|Q-\widetilde P\|_1=1$ and $\|\widehat P-\widetilde P\|_1\leq2$. Therefore, in both cases, we have $\|\widehat P-\widetilde P\|_1\leq2\|Q-\widetilde P\|_1$.
Note that $\widehat P$ remains a piece-wise polynomial, since
normalization preserves the number and degrees of the polynomial
pieces, and the fallback density $\operatorname{Unif}(I_{R,\delta})$
has at most three pieces on $\bR$.

Moreover, we can write:
\begin{align}
    \|\widetilde P-P\|_1
    &\leq
    \|\widetilde P-F\|_1+\|F-P\|_1\\
    &\leq
    \left\|
        \operatorname{Disc}_{R,\delta}(P)
        -\operatorname{Disc}_{R,\delta}(F)
    \right\|_1
    +\|\operatorname{Disc}_{R,\delta}(F)-F\|_1
    +\|F-P\|_1\\
    &\leq
    \|P-F\|_1
    +\|\operatorname{Disc}_{R,\delta}(F)-F\|_1
    +\|F-P\|_1\\
    &\leq
    2\|F-P\|_1
    +\operatorname{Disc\text{-}Err}_{R,\delta}(\cF), \label{eq:main:5}
\end{align}
where the first two inequalities follow from the triangle
inequality, the third follows because
$\operatorname{Disc}_{R,\delta}$ is a Markov operator and
does not increase $L_1$ distance
(see Lemma~\ref{lem:forward-bound}), and the last follows
from Definition~\ref{def:discretization-error}.

Finally, we can write:
\begin{align*}
    \|\widehat P-P\|_1
    &\leq
    \|\widehat P-\widetilde P\|_1+\|\widetilde P-P\|_1\\
    &\leq
    2\|Q-\widetilde P\|_1+\|\widetilde P-P\|_1\\
    &\leq
    8\alpha+8\|F-P\|_1
    +8\operatorname{Disc\text{-}Err}_{R,\delta}(\cF)+8\operatorname{Poly\text{-}Approx}_{s,d}(\cF)
    +\|\widetilde P-P\|_1\\
    &\leq
    8\alpha+10\|F-P\|_1
    +9\operatorname{Disc\text{-}Err}_{R,\delta}(\cF)
    +8\operatorname{Poly\text{-}Approx}_{s,d}(\cF),
\end{align*}
where the first inequality follows from the triangle inequality,
the second follows from \eqref{eq:main:4}, the third follows
from \eqref{eq:main:3}, and the last follows from \eqref{eq:main:5}.
Taking the infimum over $F\in\cF$ gives
\[
    \|\widehat P-P\|_1
    \leq
    10\inf_{F\in\cF}\|F-P\|_1
    +8\alpha
    +9\operatorname{Disc\text{-}Err}_{R,\delta}(\cF)
    +8\operatorname{Poly\text{-}Approx}_{s,d}(\cF).
\]

Privacy and non-interactivity follow from Theorem~\ref{thm:ldp-synthetic}, since fitting and normalization are post-processing operations. By Theorem~\ref{thm:ldp-synthetic} and Proposition~\ref{prop:deterministic-robust-general-merging}, the total running time is
\begin{align*}
    &O\left(
        \frac{t}{\alpha}\log\left(\frac{R}{\delta}\right)
        \left(n+\log\left(\frac{R}{\delta}\right)\right)
    \right)
    +\widetilde{O}\left(
        M(d+1)^{\omega+2}
        \log_+^3\left(\frac{s}{\alpha}\right)
        \left(1+\log_+\left(\frac{M}{s}\right)\right)
    \right),
\end{align*}
where $\log_+(x):=\max\{1,\log x\}$, and the $\widetilde{O}$ term hides only logarithmic factors in $M$ and $d+1$. 

For the chosen sample size, $n=\Omega(\log(R/\delta))$, so $n+\log(R/\delta)=O(n)$, $t=O(s(d+1))$, $M=O(s(d+1)/\alpha)$, and $M/s=O((d+1)/\alpha)$. Absorbing only logarithmic factors in $s$, $d+1$, and $1/\alpha$ therefore gives the stated runtime bound in the theorem.
Integrating and normalizing the polynomial pieces takes $O(s(d+1))$ additional operations, which is absorbed in the bound.
\end{proof}

\section{Omitted proofs from section~\ref{sec:apps-learning-mixtures}}\label{proofs:sec:apps-learning-mixtures}

\begin{proof}[Proof of Lemma~\ref{lem:mixture-errors}] We prove that bounding the polynomial approximation and discretization errors of mixture classes reduces to bounding the corresponding errors for the single-component classes.

Fix an arbitrary mixture $F=\sum_{i\in[m]}w_iF_i\in\mmix(\cF)$.

First we prove \textit{(1)}. By Definition~\ref{def:quantization-map},
$\operatorname{Disc}_{R,\delta}(F)
    =
\sum_{i\in[m]}w_i\operatorname{Disc}_{R,\delta}(F_i)$.
Therefore, the triangle inequality gives
\begin{align*}
    \left\|\operatorname{Disc}_{R,\delta}(F)-F\right\|_1
    &\leq
    \sum_{i\in[m]}w_i
\left\|\operatorname{Disc}_{R,\delta}(F_i)-F_i\right\|_1\\
    &\leq
    \operatorname{Disc\text{-}Err}_{R,\delta}(\cF).
\end{align*}
Taking the supremum over $F\in\mmix(\cF)$ proves one direction. The reverse inequality follows from $\cF\subseteq\mmix(\cF)$.

Next, we show \textit{(2)}. Fix $\eta>0$. For each $i\in[m]$, choose $P_i\in\cP_{s,d}$ satisfying
\begin{align*}
    \|P_i-F_i\|_1
    \leq
    \operatorname{Poly\text{-}Approx}_{s,d}(\cF)+\eta,
\end{align*}
and set $P:=\sum_{i\in[m]}w_iP_i$.
Each $P_i$ has at most $s-1$ breakpoints. Their common refinement therefore has at most $m(s-1)+1$ interval pieces, and taking weighted sums does not increase the degree of the underlying polynomials. Since $P$ is nonnegative and integrable, we have $P\in\cP_{m(s-1)+1,d}$. Moreover,
\begin{align*}
    \|P-F\|_1
    &\leq
    \sum_{i\in[m]}w_i\|P_i-F_i\|_1\leq
    \operatorname{Poly\text{-}Approx}_{s,d}(\cF)+\eta.
\end{align*}
Consequently,
\begin{align*}
    \inf_{P\in\cP_{m(s-1)+1,d}}\|P-F\|_1
    \leq
    \operatorname{Poly\text{-}Approx}_{s,d}(\cF)+\eta.
\end{align*}
Taking the supremum over $F\in\mmix(\cF)$ and letting $\eta$ tend to zero proves the result.
\end{proof}

\begin{proof}[Proof of Lemma~\ref{lem:logconcave-mixture-partition-reconstruction}] Next, we prove a bound on the discretization error of log-concave distributions.

Fix an arbitrary distribution $F\in\cL(R,\sigma_{\min},\sigma_{\max})$, and let $f$ be its density, with mean $\mu$ and standard deviation $\sigma$.

Define $g(x):=\sigma f(\mu+\sigma x)$ for $x\in\bR$.
The function $g$ is log-concave, and the change of variables $y=\mu+\sigma x$ gives
\begin{align*}
    \int_{\bR}g(x)\,dx
    &=\int_{\bR}f(y)\,dy=1,\\
    \int_{\bR}xg(x)\,dx
    &=\frac{1}{\sigma}\int_{\bR}(y-\mu)f(y)\,dy=0,\\
    \int_{\bR}x^2g(x)\,dx
    &=\frac{1}{\sigma^2}\int_{\bR}(y-\mu)^2f(y)\,dy=1.
\end{align*}
Thus, $g$ is a log-concave probability density with mean zero and variance one. By Lemma 5.5(a) of \citet{lovasz2007geometry}, $\sup_{x\in\bR}g(x)\leq1$.
Since $f(x)=\frac{1}{\sigma}g\left(\frac{x-\mu}{\sigma}\right)$, we obtain $\sup_{x\in\bR}f(x)
    =\frac{1}{\sigma}\sup_{z\in\bR}g(z)
    \leq\frac{1}{\sigma}$.

We next bound the total tail mass. For $X\sim f$, the random variable $Y=(X-\mu)/\sigma$ has density $g$. Since $\bE[Y^2]=1$, Lemma 5.17 of \citet{lovasz2007geometry} gives, for every $t>1$,
\begin{align*}
    \bP\left[|X-\mu|>t\sigma\right]
    =\bP\left[|Y|>t\right]
    \leq e^{1-t}.
\end{align*}
Since $|\mu|\leq R$ and $\sigma\leq\sigma_{\max}$, it follows that
\begin{align*}
    \bP[|X|>T]
    &\leq
    \bP\left[
        \frac{|X-\mu|}{\sigma}
        >\frac{T-R}{\sigma_{\max}}
    \right]
    \leq
    \exp\left(1-\frac{T-R}{\sigma_{\max}}\right)
    =\frac{\eta}{4}.
\end{align*}
Therefore, $F(\bR\setminus[-T,T])\leq\eta/4$.

Let $D=J+2$, and let $B_j$'s and $c_j$'s be the intervals and endpoints from Definition~\ref{def:quantization-map}, with parameters $(T,\delta)$.
For each $j\in\{2,\cdots,D-1\}$, let $h_j=c_j-c_{j-1}\leq\delta$ and
\begin{align*}
    \bar f_j:=\frac{1}{h_j}\int_{B_j}f(y)\,dy.
\end{align*}
For every $x\in B_j$, we have
\begin{align*}
    |f(x)-\bar f_j|
    &\leq
    \frac{1}{h_j}\int_{B_j}|f(x)-f(y)|\,dy\\
    &\leq
    \sup_{y\in B_j}f(y)-\inf_{y\in B_j}f(y).
\end{align*}

Since $f$ is nonnegative and unimodal, the sum of its binwise differences between the supremum and infimum is bounded by its total increase plus its total decrease, each of which is at most $\sup_{x\in\bR}f(x)$. Therefore,
\begin{align*}
    \sum_{j=2}^{D-1}
    \left(\sup_{x\in B_j}f(x)-\inf_{x\in B_j}f(x)\right)
    \leq2\sup_{x\in\bR}f(x)
    \leq\frac{2}{\sigma}.
\end{align*}
Integrating over $x\in B_j$ and summing over the internal intervals yields
\begin{align*}
    \sum_{j=2}^{D-1}\int_{B_j}|f(x)-\bar f_j|\,dx
    &\leq
    \delta
    \sum_{j=2}^{D-1}
    \left(\sup_{x\in B_j}f(x)-\inf_{x\in B_j}f(x)\right)\leq
    \frac{2\delta}{\sigma}
    \leq\frac{2\delta}{\sigma_{\min}}.
\end{align*}
The original and discretized measures restricted to the two tails each have total mass $F(\bR\setminus[-T,T])$. Their $L_1$ difference is therefore at most twice this mass. Combining the two bounds gives
\begin{align*}
    \left\|\operatorname{Disc}_{T,\delta}(F)-F\right\|_1
    &\leq
    \frac{2\delta}{\sigma_{\min}}+\frac{\eta}{2}=\frac{4T}{J\sigma_{\min}}+\frac{\eta}{2}
    \leq\eta.
\end{align*}
Taking the supremum over $F$ proves the result.
\end{proof}

\begin{lemma}[Piece-wise polynomial approximation of log-concave distributions, Lemma 27 of ~\citet{chan2014efficient}]
\label{lem:cdss-logconcave-polynomial-approximation}
For every $\alpha\in(0,1)$, there exists an integer $s_\alpha=\widetilde{O}(1/\sqrt{\alpha})$ such that, for every log-concave probability density $f$ over $\bR$, there exists a (normalized) piece-wise polynomial $p\in\cP_{s_\alpha,1}$ satisfying $\|f-p\|_1\leq\alpha$. Here, $\widetilde{O}$ hides only logarithmic factors in $1/\alpha$.
\end{lemma}

\begin{proof}[Proof of Lemma~\ref{lem:logconcave-polynomial-approximation}] We prove a bound on the piece-wise polynomial approximation of log-concave distributions.

Let $s_\alpha$ be the integer from Lemma~\ref{lem:cdss-logconcave-polynomial-approximation}. For every $f\in\cL(R,\sigma_{\min},\sigma_{\max})$, that lemma gives a normalized $P\in\cP_{s_\alpha,1}$ satisfying $\|P-f\|_1\leq\alpha$. Therefore,
$\inf_{P\in\cP_{s_\alpha,1}}\|P-f\|_1
    \leq\alpha$.
Taking the supremum over $f\in\cL(R,\sigma_{\min},\sigma_{\max})$ proves the result.
\end{proof}

\begin{proof}[Proof of Theorem~\ref{thm:ldp-logconcave-mixture-learning}] Putting together, we obtain our upper bound for LDP agnostic learning log-concave distributions.

Let $\cF=\mmix\left(\cL(R,\sigma_{\min},\sigma_{\max})\right)$ be the class of log-concave mixtures.
Set $\rho=\frac{\alpha}{25},
    s_0=m(s_\rho-1)+1,
    d=1$,
where $s_\rho$ is the integer from Lemma~\ref{lem:logconcave-polynomial-approximation} at accuracy $\rho$. Thus, $s_0=\widetilde{O}(m/\sqrt{\rho})$.

Also let the discretization parameters $T=R+\sigma_{\max}\left(1+\log\left(\frac{4}{\rho}\right)\right)$, $J=\left\lceil\frac{8T}{\rho\sigma_{\min}}\right\rceil$, and
    $\delta=\frac{2T}{J}$ be defined as per Lemma~\ref{lem:logconcave-mixture-partition-reconstruction}.
Since $T>0$ and $J\geq4$, we have $0<\delta\leq T/2$. We can apply Theorem~\ref{main:cor} with class $\cF$, discretization parameters $(T,\delta)$, privacy parameter $\eps$, failure probability $\beta$, accuracy parameter $\rho$, and polynomial parameters $(s_0,d)$, using sample size
$n=O\left(
        \frac{s_0^2(d+1)^2}{\eps^2\rho^2}
        \log^2\left(\frac{T}{\delta}\right)
        \log\left(\frac{T}{\beta\delta}\right)
    \right)$.
The output then satisfies
\begin{align*}
    \|\widehat P-P\|_1
    \leq
    10\mathrm{OPT}^*+8\rho
    +9\operatorname{Disc\text{-}Err}_{T,\delta}(\cF)
    +8\operatorname{Poly\text{-}Approx}_{s_0,d}(\cF),
\end{align*}
with probability at least $1-\beta$, where $\mathrm{OPT}^*=\inf_{F\in\cF}\|P-F\|_1$.

By Lemmas~\ref{lem:logconcave-polynomial-approximation} and~\ref{lem:mixture-errors},
\begin{align*}
    \operatorname{Poly\text{-}Approx}_{s_0,d}(\cF)
    &\leq
    \operatorname{Poly\text{-}Approx}_{s_\rho,1}
    \left(\cL(R,\sigma_{\min},\sigma_{\max})\right)
    \leq\rho.
\end{align*}
Moreover, Lemmas~\ref{lem:logconcave-mixture-partition-reconstruction} and~\ref{lem:mixture-errors} give
\begin{align*}
    \operatorname{Disc\text{-}Err}_{T,\delta}(\cF)
    &=
    \operatorname{Disc\text{-}Err}_{T,\delta}
    \left(\cL(R,\sigma_{\min},\sigma_{\max})\right)
    \leq\rho.
\end{align*}

Substituting these two bounds gives
\begin{align*}
    \|\widehat P-P\|_1
    &\leq
    10\mathrm{OPT}^*+8\rho+9\rho+8\rho=10\mathrm{OPT}^*+\alpha.
\end{align*}

Finally, $s_0=\widetilde{O}(m/\sqrt{\rho})$, $d=1$, and our parameter choices give $J=O\left(
        \frac{R+\sigma_{\max}}{\rho\sigma_{\min}}
        \log\left(\frac{4}{\rho}\right)
    \right)$, and $\log\left(\frac{T}{\delta}\right)
    =O\left(
    \log\left(\frac{2(R+\sigma_{\max})}{\alpha\sigma_{\min}}\right)
    \right)$.
Substituting the chosen parameters into the sample and running-time bounds of Theorem~\ref{main:cor} gives the stated complexities, suppressing only logarithmic factors in $1/\alpha$ in the sample bound, and in $m$ and $1/\alpha$ in the running time. The output has $\widetilde{O}(m/\sqrt{\alpha})$ polynomial pieces and degree at most $1$. Privacy and non-interactivity also follow from Theorem~\ref{main:cor}.
\end{proof}

\section{Upper bound for GMMs}
\label{sec:efficient-nonproper-gmm}

In this section, we consider agnostic LDP learning of GMMs. Although GMMs are mixtures of log-concave distributions, Gaussian components admit a more specialized piece-wise polynomial approximation, whose number of pieces is independent of the accuracy parameter $\alpha$. This allows us to obtain an optimal sample dependence on the accuracy parameter. The class of bounded-parameter Gaussian distributions is defined as follows.

\begin{definition} [Bounded Gaussians] \label{def:general_gaussian}
    Let $R\geq 0$, and $0<\sigma_{min}\leq \sigma_{max}<\infty$ be parameters. Then the class of bounded Gaussians over the real line is defined by:
    \[
\cG(R,\sigma_{min},\sigma_{max})=\left\{\cN(\mu, \sigma^2): |\mu| \leq R, \sigma \in [\sigma_{min}, \sigma_{max}]\right\}.
    \]
\end{definition}

The following theorem gives the resulting guarantee for efficient non-interactive LDP agnostic learning of GMMs. 
\begin{theorem}[Efficient LDP Agnostic Learning of GMMs]
\label{thm:ldp-gmm-learning}
Under the same assumptions of Theorem~\ref{thm:ldp-logconcave-mixture-learning}, there exists a non-interactive $\eps$-LDP $(10,\alpha,\beta)$-agnostic learner for the class $\mmix(\cG(R,\sigma_{\min},\sigma_{\max}))$, that uses $
    n
    =O\left(
        \frac{m^2\log^2(2/\alpha)}{\eps^2\alpha^2}
        \log^2\left(
            \frac{2(R+\sigma_{\max})}{\alpha\sigma_{\min}}
        \right)
        \log\left(
            \frac{2(R+\sigma_{\max})}{\alpha\beta\sigma_{\min}}
        \right)
    \right)$
many samples and its output 
has a nonnegative piece-wise polynomial density with $O(m)$ pieces and degree $O(\log(2/\alpha))$. Moreover, the total running time, is
$\widetilde{O}\left(
        \frac{nm}{\alpha}
        \log\left(
            \frac{2(R+\sigma_{\max})}{\alpha\sigma_{\min}}
        \right)
    \right)$,
where $\widetilde{O}$ hides only logarithmic factors in $m$ and $1/\alpha$.
\end{theorem}

\begin{proof}
Let $\cF=\mmix\left(\cG(R,\sigma_{\min},\sigma_{\max})\right)$ be the class of GMMs.
Set $\rho=\frac{\alpha}{25},
    s_0=2m+1,
    d=\left\lceil C\log\left(\frac{2}{\rho}\right)\right\rceil$,
where $C>0$ is the universal constant from Lemma~\ref{lem:g-polynomial-approximation}.

Also let the discretization parameters $T=R+\sigma_{\max}\left(1+\log\left(\frac{4}{\rho}\right)\right), J=\left\lceil\frac{8T}{\rho\sigma_{\min}}\right\rceil$, and
    $\delta=\frac{2T}{J}$ be defined as per Lemma~\ref{lem:gmm-partition-reconstruction}.
Since $T>0$ and $J\geq4$, we have $0<\delta\leq T/2$, we can apply Theorem~\ref{main:cor} with class $\cF$, discretization parameters $(T,\delta)$, privacy parameter $\eps$, failure probability $\beta$, accuracy parameter $\rho$, and polynomial parameters $(s_0,d)$, using sample size $n=O\left(
        \frac{s_0^2(d+1)^2}{\eps^2\rho^2}
        \log^2\left(\frac{T}{\delta}\right)
        \log\left(\frac{T}{\beta\delta}\right)
    \right)$.
The output then satisfies
\begin{align*}
    \|\widehat P-P\|_1
    \leq
    10\mathrm{OPT}^*+8\rho
    +9\operatorname{Disc\text{-}Err}_{T,\delta}(\cF)
    +8\operatorname{Poly\text{-}Approx}_{s_0,d}(\cF).
\end{align*}
with probability at least $1-\beta$, where $\mathrm{OPT}^*:=\inf_{G\in\cF}\|P-G\|_1$.

By Lemmas~\ref{lem:g-polynomial-approximation} and~\ref{lem:mixture-errors},
\begin{align*}
    \operatorname{Poly\text{-}Approx}_{s_0,d}(\cF)
    &\leq
    \operatorname{Poly\text{-}Approx}_{3,d}
    \left(\cG(R,\sigma_{\min},\sigma_{\max})\right)
    \leq\rho.
\end{align*}
Moreover, Lemmas~\ref{lem:gmm-partition-reconstruction} and~\ref{lem:mixture-errors} give
\begin{align*}
    \operatorname{Disc\text{-}Err}_{T,\delta}(\cF)
    &=
    \operatorname{Disc\text{-}Err}_{T,\delta}
    \left(\cG(R,\sigma_{\min},\sigma_{\max})\right)
    \leq\rho.
\end{align*}

Substituting these two bounds gives
\begin{align*}
    \|\widehat P-P\|_1
    &\leq
    10\mathrm{OPT}^*+8\rho+9\rho+8\rho=10\inf_{G\in\cF}\|P-G\|_1+\alpha.
\end{align*}

Finally, our choices of $T$ and $\delta$ give
$\log\left(\frac{T}{\delta}\right)
    =O\left(
        \log\left(\frac{2(R+\sigma_{\max})}{\alpha\sigma_{\min}}\right)
    \right)$.
Substituting the chosen parameters into the sample and running-time bounds of Theorem~\ref{main:cor} gives the stated complexities, suppressing only logarithmic factors in $m$ and $1/\alpha$ in the running time. The output has $O(m)$ polynomial pieces and degree $O(\log(2/\alpha))$. Privacy and non-interactivity also follow from Theorem~\ref{main:cor}.
\end{proof}

\begin{lemma}[Discretization of Bounded-Parameter Gaussians]
\label{lem:gmm-partition-reconstruction}
Let $\eta\in(0,1)$, $R\geq0$, and $0<\sigma_{\min}\leq\sigma_{\max}<\infty$ be parameters. Let
$T=R+\sigma_{\max}\left(1+\log\left(\frac{4}{\eta}\right)\right)$,
$J=\left\lceil\frac{8T}{\eta\sigma_{\min}}\right\rceil$,
and $\delta=\frac{2T}{J}$.
Then, it holds that $\operatorname{Disc\text{-}Err}_{T,\delta}
    \left(
        \cG(R,\sigma_{\min},\sigma_{\max})
    \right)
    \leq\eta$.
\end{lemma}
\begin{proof}
Every Gaussian density is log-concave. Therefore,
$\cG(R,\sigma_{\min},\sigma_{\max})
    \subseteq
    \cL(R,\sigma_{\min},\sigma_{\max})$.
    
By Lemma~\ref{lem:logconcave-mixture-partition-reconstruction}, we get
$\operatorname{Disc\text{-}Err}_{T,\delta}
    \left(\cG(R,\sigma_{\min},\sigma_{\max})\right)
    \leq
    \operatorname{Disc\text{-}Err}_{T,\delta}
    \left(\cL(R,\sigma_{\min},\sigma_{\max})\right)
    \leq\eta$.
\end{proof}

\begin{lemma}[Piece-wise polynomial approximation of Gaussians, Lemma 36 of ~\citet{chan2014efficient}]
\label{lem:cdss-g-polynomial-approximation}
There exists a universal constant $C>0$ such that, for every $\alpha\in(0,1)$ and every Gaussian distribution
$g=\cN(\mu,\sigma^2)$ with $\mu\in\bR$ and $\sigma>0$, there exists a (normalized) piece-wise polynomial $p\in \cP_{s,d}$ with $s=3$, and $d=\lceil C \log(\frac{2}{\alpha})\rceil$ satisfying $\|g-p\|_1\leq\alpha$.
\end{lemma}

 We use the Gaussian approximation bound of \citet{chan2014efficient} to obtain the following lemma.

\begin{lemma}[Piece-wise polynomial approximation of Gaussians]
\label{lem:g-polynomial-approximation}
There exists a universal constant $C>0$ such that the following holds. Let $\alpha\in(0,1)$, $R\geq0$, and $0<\sigma_{\min}\leq\sigma_{\max}<\infty$ be parameters. Set $s=3$, and $d=\lceil C\log(\frac{2}{\alpha})\rceil$. Then,
it holds that $\operatorname{Poly\text{-}Approx}_{s,d}
    \left(\cG(R,\sigma_{\min},\sigma_{\max})\right)
    \leq\alpha$.
\end{lemma}

\begin{proof}
Let $C>0$ be the universal constant from Lemma~\ref{lem:cdss-g-polynomial-approximation}. For every $g\in\cG(R,\sigma_{\min},\sigma_{\max})$, that lemma gives a normalized $P\in\cP_{s,d}$ satisfying $\|P-g\|_1\leq\alpha$. Therefore, $\inf_{P\in\cP_{s,d}}\|P-g\|_1
    \leq\alpha$.
Taking the supremum over $g\in\cG(R,\sigma_{\min},\sigma_{\max})$ proves the result.
\end{proof}

\section{Lower bound for LDP-learning of  GMMs}\label{sec:lower-bound}

In this section, we prove a lower bound for non-interactive LDP learning of GMMs. When the ranges of the component means and variances are
sufficiently large, the lower bound matches the dependence on the number of mixture components, accuracy and privacy parameters
of our upper bound in Theorem~\ref{thm:ldp-gmm-learning}
 up to logarithmic factors. Since every Gaussian distribution is log-concave, this lower bound also applies to mixtures of log-concave distributions with the same component parameter bounds.

In particular, our lower bound improves on the standard packing lower bounds in pure differential privacy~\citep{bun2019private}: packings for mixtures of $m$ components usually have sizes exponential in $m$, yielding lower bounds of $\Omega(m)$, whereas we obtain a lower bound of $\Omega(m^2)$.

We obtain this result by reducing LDP estimation of sparse discrete distributions to GMM learning and applying the lower bound of \citet{acharya2020estimating}. The main idea is to associate each element of a finite domain with a known Gaussian component, so that a sparse probability vector becomes the vector of mixture weights. We choose components separated in location or scale, with each component placing most of its probability mass in its own disjoint region. This separation ensures that an accurate estimate of the mixture yields an accurate estimate of the sparse probability vector. The reduction preserves local privacy, so a GMM learner would also solve the sparse estimation problem with the same number of samples. The lower bound for sparse estimation therefore gives the desired GMM lower bound.

\begin{theorem}[GMM Lower Bound]
\label{thm:lower-bound}
Let $\eps\in(0,1]$ and $\alpha\in(0,\alpha_0]$ for a sufficiently small universal constant $\alpha_0>0$.
Let $m\geq2$ be an integer, $R\geq0$, and $0<\sigma_{\min}\leq\sigma_{\max}<\infty$.
Assume that
$\Lambda:=R/\sigma_{\min}
+\log(\sigma_{\max}/\sigma_{\min})\geq C_0m$,
where $C_0>0$ is a sufficiently large universal constant.
Then, any non-interactive $\eps$-LDP algorithm that learns
$\mmix(\cG(R,\sigma_{\min},\sigma_{\max}))$ to total variation distance $\alpha$ with probability at least $0.9$ requires
$n=\Omega\left(
        \frac{m^2}{\eps^2\alpha^2}
        \log\left(e+\frac{\Lambda}{m}\right)
    \right)$
samples. This lower bound holds even for non-proper learners.
\end{theorem}

Since every Gaussian density is log-concave,
Theorem~\ref{thm:lower-bound} immediately implies the same
lower bound for learning mixtures of log-concave distributions $\mmix(\cL(R,\sigma_{\min},\sigma_{\max}))$.

\begin{theorem}[Sparse distribution estimation lower bound \citep{acharya2020estimating}]
\label{thm:lower-sparse}
There exists a universal constant $c_0>0$ such that the following holds.
Let $\eps\in(0,1]$, $\alpha\in(0,c_0]$, and let $k,s$ be integers
satisfying $2\leq s\leq k/100$.
Let $\cS_{k,s}:=\{p\in\Delta_k:\|p\|_0\leq s\}$ be the class of
distributions supported on at most $s$ elements, where
$\Delta_k=\left\{(w_1,\ldots,w_k)\in\bR_{\geq0}^k:
    \sum_{i\in[k]}w_i=1\right\}$.
Then, any non-interactive $\eps$-LDP algorithm that learns
$\cS_{k,s}$ to total variation distance $\alpha$ with probability
at least $0.9$ requires
\begin{align*}
    n=\Omega\left(
    \frac{s^2\log(k/s)}{\eps^2\alpha^2}
    \right)
\end{align*}
samples. The output is not required to be $s$-sparse.
\end{theorem}

\begin{proof}[Proof of Theorem~\ref{thm:lower-bound}] Now, we give our lower bound proof for LDP learning of GMMs.

Define parameters $K_\mu=1+\left\lfloor\frac{R}{\sigma_{\min}}\right\rfloor, 
    K_\sigma=1+\left\lfloor\frac{\log(\sigma_{\max}/\sigma_{\min})}{\log16}
    \right\rfloor,
    K=\max\{K_\mu,K_\sigma\}$,  and $M=\min\{m,K\}$.

Taking $C_0\geq100(1+\log16)$, the assumption gives
$K
    \geq
    \max\left\{
        \frac{R}{\sigma_{\min}},
        \frac{\log(\sigma_{\max}/\sigma_{\min})}{\log16}
    \right\}
    \geq
    \frac{\Lambda}{1+\log16}
    \geq100m$.
Thus, $M=m$ and $2\leq M\leq K/100$, as required by the sparse estimation lower bound in 
Theorem~\ref{thm:lower-sparse}.

We construct $K$ Gaussian components $g_1,\ldots,g_K\in\cG(R,\sigma_{\min},\sigma_{\max})$ and disjoint measurable sets $B_1,\ldots,B_K\subset\bR$ such that $g_j(B_j)\geq2/3$ for all $j\in[K]$.
Let $\Phi$ denote the standard normal CDF. We use the larger of the following two constructions.

\textit{(1) Location Gaussians.}
For $j\in[K_\mu]$, set
$\mu_j=(2j-K_\mu-1)\sigma_{\min}, g_j=\cN(\mu_j,\sigma_{\min}^2)$, and $B_j=[\mu_j-\sigma_{\min},\mu_j+\sigma_{\min})$. Then, the means satisfy $|\mu_j|\leq(K_\mu-1)\sigma_{\min}\leq R$. The set $B_j$'s are disjoint, and
\begin{align*}
    g_j(B_j)=2\Phi(1)-1>\frac{2}{3}.
\end{align*}

\textit{(2) Scale Gaussians.}
For $j\in[K_\sigma]$, set
$\sigma_j=16^{j-1}\sigma_{\min},
    g_j=\cN(0,\sigma_j^2)$, and
    $B_j=\{x:\sigma_j/4\leq|x|<4\sigma_j\}$.
We have $\sigma_j\in[\sigma_{\min},\sigma_{\max}]$. Then, the sets $B_j$'s are disjoint, and
\begin{align*}
    g_j(B_j)=2\bigl(\Phi(4)-\Phi(1/4)\bigr)>\frac{2}{3}.
\end{align*}

For $p\in\Delta_K$, write $f_p:=\sum_{j\in[K]}p_jg_j$, and let $\gamma=1/3$.
Since the sets are disjoint, $\sum_{i\neq j}g_j(B_i)\leq1-g_j(B_j)$.
Thus, the reverse triangle inequality gives, for any $p,q\in\Delta_K$,
\begin{align*}
    \dtv(f_p,f_q)
    &\geq\frac{1}{2}\sum_{i\in[K]}
    \left|\sum_{j\in[K]}(p_j-q_j)g_j(B_i)\right|\\
    &\geq\frac{1}{2}\sum_{j\in[K]}|p_j-q_j|
    \left(g_j(B_j)-\sum_{i\neq j}g_j(B_i)\right)\\
    &\geq\frac{1}{2}\sum_{j\in[K]}|p_j-q_j|
    \bigl(2g_j(B_j)-1\bigr)\\
    &\geq\gamma\dtv(p,q).
\end{align*}

We reduce learning $\cS_{K,M}$ to learning the Gaussian mixture class.
Given a discrete sample $Y=j$, the corresponding user independently draws $X\sim g_j$ and applies the local randomizer of the Gaussian mixture learner to $X$.
For any such randomizer $Q$ and measurable message event $A$, integrating its privacy inequality gives
\begin{align*}
    \int Q(A\mid x)\,dg_i(x)
    \leq e^\eps\int Q(A\mid x)\,dg_j(x)
\end{align*}
for all $i,j\in[K]$. Hence, the resulting protocol remains non-interactive $\eps$-LDP.
If $Y\sim p\in\cS_{K,M}$, then $X\sim f_p$, and $f_p$ has at most $M\leq m$ Gaussian components.

Let $\widehat f$ be the output of the learner, and define
$\widehat p\in\operatorname*{arg\,min}_{q\in\cS_{K,M}}
    \dtv(\widehat f,f_q)$.
A minimizer exists because $\cS_{K,M}$ is compact and the objective is continuous in $q$.

On the event $\dtv(\widehat f,f_p)\leq\alpha$, we have
\begin{align*}
    \gamma\dtv(\widehat p,p)
    &\leq\dtv(f_{\widehat p},f_p)\\
    &\leq\dtv(f_{\widehat p},\widehat f)+\dtv(\widehat f,f_p)\\
    &\leq2\alpha.
\end{align*}
This decoding is post-processing and returns an $M$-sparse probability vector.
The resulting discrete learner has accuracy $2\alpha/\gamma$ and succeeds with probability at least $0.9$.
Choosing $\alpha_0$ so that $2\alpha_0/\gamma\leq c_0$, Theorem~\ref{thm:lower-sparse} therefore gives
$n=\Omega\left(
        \frac{m^2\log(K/m)}{\eps^2\alpha^2}
    \right)$.
Finally, since $K=\Theta(\Lambda)$ and $K/m\geq100$, $\log(K/m)=\Theta\left(\log\left(e+\frac{\Lambda}{m}\right)\right)$,
which gives the stated bound.
\end{proof}

\section{Discussion on proper learning and connections to Private Hypothesis Selection}
\label{sec:proper-learning}

The results in the preceding sections concern non-proper learning. In this section, we convert these learners into proper learners by projecting their outputs onto a finite cover of the target class. This preserves the sample complexity up to a constant rescaling of the accuracy parameter, at the cost of a larger agnostic approximation factor and potentially inefficient computation.

A generic alternative for non-efficient proper learning a class $\cF$ is to apply locally private hypothesis selection to a finite $\alpha$-cover $\cC\subseteq \cF$. The selector returns a member of $\cC$, so its output belongs to the target class. After a constant rescaling of $\alpha$ to account for the cover error, the interactive algorithm of \citet{pour2024sample} gives sample complexity
$\widetilde{O}(|\cC|/(\eps^2\alpha^2))$
with agnostic approximation factor $9$, while the non-interactive algorithm of \citet{kamath2025query} gives
$\widetilde{O}(|\cC|^{3/2}/(\eps^2\alpha^2))$
with agnostic approximation factor $13$. Here, $\widetilde{O}$ suppresses logarithmic factors in $|\cC|$ and $1/\beta$.

However, such covers can be exponentially large. For example,
for GMMs, the cover size is roughly
$\widetilde{O}\left(
    \left(\frac{mR}{\alpha^3\sigma_{\min}}\right)^m
\right)$. Applying LDP hypothesis selection to such a
cover yields sample complexity bound that is exponential
in $m$ and polynomial in range parameters $R/\alpha\sigma_{\min}$. In comparison, our proper learning result in Corollary~\ref{cor:ldp-proper-gmm-learning} preserves the sample bound of our non-proper learner, with quadratic dependence on $m$ (up to logarithmic factors) and only polylogarithmic dependence on the component range parameters. The cover size affects the computation required for the final projection, but does not enter the sample bound. The resulting agnostic approximation factor is $21$. 

For an arbitrary class of distributions $\cF$, an $\Omega(|\cF|/(\eps^2\alpha^2))$ worst-case sample lower bound holds even for interactive LDP hypothesis selection~\citep{gopi2020locally,pour2024sample}. This worst-case bound prevents a generic improvement based only on $|\cF|$; it does not impose the same dependence for structured distribution classes.

\begin{lemma}[From non-proper to proper learning]
\label{thm:nonproper-to-proper}
Let $C\geq1$, $\alpha\in(0,1/4)$, and $\beta\in(0,1)$.
Let $\cF$ be a class of distributions with a given
finite $\alpha$-cover $\cC\subseteq\cF$ in total variation distance.
Let $\cA$ be a non-proper $(C,\alpha,\beta)$-agnostic learner
for $\cF$.
Suppose there exists an oracle that, given any $F\in\cF$ and
any distribution $Q$, approximates $\dtv(Q,F)$ with additive
 error at most $\alpha/2$ in time $T_{\mathrm{dist}}$. Then, the following holds:
\begin{quote}
    There exists a proper $(2C+1,4\alpha,\beta)$-agnostic
learner for $\cF$ that outputs a member of $\cC$, uses the same
number of samples, and takes
$O(|\cC|(T_{\mathrm{dist}}+1))$ additional time.
Moreover, if $\cA$ is non-interactive $\eps$-LDP, then the
resulting proper learner is also non-interactive $\eps$-LDP.
\end{quote}
\end{lemma}

The proof, given later in this section, follows by approximately projecting the non-proper learner's output onto $\cC$ and applying the triangle inequality.

The next lemma provides such a finite cover for the class of bounded-parameter Gaussian mixtures. Its proof appears later in this section.

\begin{lemma}\label{lemma:cover-gmm}
Let $\alpha\in(0,1)$, $R\geq0$, $m\geq1$ be an integer,
$0<\sigma_{\min}\leq\sigma_{\max}<\infty$, and $c>0$ be a universal constant.
Then the class
$\mmix(\cG(R,\sigma_{\min},\sigma_{\max}))$
admits an $\alpha$-cover in total variation distance of size at most
$\left[
        \frac{cm}{\alpha}
        \left(1+\frac{R}{\alpha\sigma_{\min}}\right)
        \left(
            1+\frac{\log(\sigma_{\max}/\sigma_{\min})}{\alpha}
        \right)
    \right]^m$.
\end{lemma}

Combining the proper conversion from~\ref{thm:nonproper-to-proper} with Lemma~\ref{lemma:cover-gmm} and Theorem~\ref{thm:ldp-gmm-learning} gives the following corollary.

\begin{corollary}[LDP Agnostic Proper Learning of GMMs]
\label{cor:ldp-proper-gmm-learning}
Under the same assumptions of Theorem~\ref{thm:ldp-gmm-learning}, there exists a non-interactive $\eps$-LDP $(21,\alpha,\beta)$-agnostic proper learner for the class of $\mmix(\cG(R,\sigma_{\min},\sigma_{\max}))$ that uses $O\left(
        \frac{m^2\log^2(2/\alpha)}{\eps^2\alpha^2}
        \log^2\left(
            \frac{2(R+\sigma_{\max})}{\alpha\sigma_{\min}}
        \right)
        \log\left(
            \frac{2(R+\sigma_{\max})}{\alpha\beta\sigma_{\min}}
        \right)
    \right)$ samples.
The algorithm is not claimed to be computationally efficient.
\end{corollary}

\begin{proof}[Proof of Lemma~\ref{thm:nonproper-to-proper}] Now, we prove how to obtain a proper learner from a non-proper learner.

Run $\cA$ to obtain a distribution $\widehat P$.
Approximate each distance $\dtv(\widehat P,G)$, for $G\in\cC$,
to additive error at most $\alpha/2$, and return a member
$\widetilde P\in\cC$ with the smallest estimated distance,
breaking ties according to a fixed ordering of $\cC$.
This ensures that
\[
    \dtv(\widehat P,\widetilde P)
    \leq
    \min_{G\in\cC}\dtv(\widehat P,G)+\alpha.
\]

Let $P$ be the sampling distribution.
With probability at least $1-\beta$, the guarantee of $\cA$ gives
\[
    \dtv(\widehat P,P)
    \leq C\inf_{F\in\cF}\dtv(P,F)+\alpha.
\]
Condition on this event. Fix any $F^\star\in\cF$.
Since $\cC$ is an $\alpha$-cover, there exists $G^\star\in\cC$
such that $\dtv(F^\star,G^\star)\leq\alpha$.
Therefore,
\begin{align*}
    \dtv(\widetilde P,P)
    &\leq
    \dtv(\widetilde P,\widehat P)+\dtv(\widehat P,P)\\
    &\leq
    \dtv(G^\star,\widehat P)+\dtv(\widehat P,P)+\alpha\\
    &\leq
    \dtv(F^\star,\widehat P)+\dtv(\widehat P,P)+2\alpha\\
    &\leq
    \dtv(F^\star,P)+2\dtv(\widehat P,P)+2\alpha\\
    &\leq
    (2C+1)\dtv(P,F^\star)+4\alpha.
\end{align*}
Taking the infimum over $F^\star\in\cF$ proves the accuracy
guarantee. The additional errors account for the cover and
the approximate minimization.

The conversion computes $|\cC|$ distance approximations
and makes $O(|\cC|)$ additional comparisons, giving
$O(|\cC|(T_{\mathrm{dist}}+1))$ additional running time.
It uses only the output of $\cA$, so it requires no additional
samples and preserves privacy and non-interactivity by
post-processing.

For a target additive error $\tau\in(0,1)$, applying this
construction with $\alpha=\tau/4$ and a cover of radius
$\tau/4$ gives a proper $(2C+1,\tau,\beta)$-agnostic learner
using $n_{\cA}(\tau/4,\beta)$ samples.
\end{proof}

Here, we explain the construction of a cover for bounded-parameter GMMs.

\begin{claim}\label{lemma:weightcover}
Let $\alpha\in(0,1]$.
There is an $\alpha$-cover for the $(m-1)$-dimensional probability simplex
$\Delta_m=\left\{(w_1,w_2,\ldots,w_m)\in\bR_{\geq0}^m\,:\,
\sum_{i\in[m]}w_i=1\right\}$
with respect to $\ell_\infty$ of size at most
$\lceil1/\alpha\rceil^m$.
\end{claim}

\begin{proof}
Partition $[0,1]$ into $\lceil1/\alpha\rceil$ intervals of length at most $\alpha$, and use their Cartesian products to partition $[0,1]^m$.
For every cell that intersects $\Delta_m$, select one point from the intersection.
Any two points in the same cell have $\ell_\infty$ distance at most $\alpha$, so the selected points form an $\alpha$-cover of $\Delta_m$.
There are at most $\lceil1/\alpha\rceil^m$ cells.
\end{proof}

\begin{proof}[Proof of Lemma~\ref{lemma:cover-gmm}] We prove the cover construction for GMMs.

Let $f=\sum_{i\in[m]}w_i\cN(\mu_i,\sigma_i^2)$ be an arbitrary member of the class.
We construct a cover by discretizing the mixture weights, means, and standard deviations.

For the weights, let $C_w$ be an $\alpha/m$-cover of $\Delta_m$ in $\ell_\infty$, of size at most $\lceil m/\alpha\rceil^m$, as given by Claim~\ref{lemma:weightcover}.

For the means, let $\cC_\mu$ be a grid on $[-R,R]$ containing both endpoints, with consecutive points at distance at most $\alpha\sigma_{\min}/2$.
When $R=0$, take $\cC_\mu=\{0\}$.
The grid can be chosen to have at most
$1+\lceil4R/(\alpha\sigma_{\min})\rceil$ points, and every $\mu\in[-R,R]$ has a grid point $\tilde\mu$ satisfying
$|\mu-\tilde\mu|\leq\alpha\sigma_{\min}/4$.

For the standard deviations, take
\begin{align*}
    \cC_\sigma
    =
    \left\{
        \sigma_{\min}(1+\alpha/4)^j:
        j=0,\ldots,
        \left\lfloor
            \frac{\log(\sigma_{\max}/\sigma_{\min})}
                 {\log(1+\alpha/4)}
        \right\rfloor
    \right\}.
\end{align*}
For every $\sigma\in[\sigma_{\min},\sigma_{\max}]$, there is a grid point $\tilde\sigma$ such that
$\tilde\sigma\leq\sigma<(1+\alpha/4)\tilde\sigma$, and hence
$|\log(\sigma/\tilde\sigma)|\leq\alpha/4$.

Let $\cC$ contain all mixtures formed from weights in $C_w$, means in $\cC_\mu$, and standard deviations in $\cC_\sigma$.
Every member of $\cC$ belongs to the target class.

For univariate Gaussians, varying the mean and then the standard deviation gives
\begin{align*}
    \dtv\left(
        \cN(\mu,\sigma^2),
        \cN(\tilde\mu,\tilde\sigma^2)
    \right)
    \leq
    \frac{|\mu-\tilde\mu|}{2\sigma}
    +\left|\log\left(\frac{\sigma}{\tilde\sigma}\right)\right|.
\end{align*}

Choose the corresponding grid parameters for $f$, and write
$\tilde f=\sum_{i\in[m]}\tilde w_i
\cN(\tilde\mu_i,\tilde\sigma_i^2)\in\cC$.
Then
\begin{align*}
    \dtv(f,\tilde f)
    &\leq
    \frac{1}{2}\sum_{i\in[m]}|w_i-\tilde w_i|
    +\sum_{i\in[m]}\tilde w_i
        \dtv\left(
            \cN(\mu_i,\sigma_i^2),
            \cN(\tilde\mu_i,\tilde\sigma_i^2)
        \right)\\
    &\leq
    \frac{\alpha}{2}
    +\sum_{i\in[m]}\tilde w_i
        \left(
            \frac{\alpha\sigma_{\min}/4}{2\sigma_{\min}}
            +\frac{\alpha}{4}
        \right)\\
    &=\frac{\alpha}{2}+\frac{3\alpha}{8}
    \leq\alpha.
\end{align*}
Finally,
\begin{align*}
    |\cC|
    &\leq |C_w||\cC_\mu|^m|\cC_\sigma|^m\\
    &\leq
    \left\lceil\frac{m}{\alpha}\right\rceil^m
    \left(1+\left\lceil
        \frac{4R}{\alpha\sigma_{\min}}
    \right\rceil\right)^m
    \left(1+\left\lfloor
        \frac{\log(\sigma_{\max}/\sigma_{\min})}
             {\log(1+\alpha/4)}
    \right\rfloor\right)^m\\
    &\leq
    \left[
        \frac{c m}{\alpha}
        \left(1+\frac{R}{\alpha\sigma_{\min}}\right)
        \left(1+\frac{\log(\sigma_{\max}/\sigma_{\min})}{\alpha}\right)
    \right]^m
\end{align*}
for a universal constant $c>0$.
\end{proof}

\end{document}